\documentclass{article}

\usepackage{microtype}
\usepackage{graphicx}
\usepackage{subfigure}
\usepackage{booktabs} 

\usepackage{hyperref}

\usepackage[accepted]{icml2024}

\usepackage{amsmath}
\usepackage{amssymb}
\usepackage{mathtools}
\usepackage{amsthm}

\usepackage[algo2e,ruled]{algorithm2e}

\usepackage[capitalize,noabbrev]{cleveref}

\theoremstyle{plain}
\newtheorem{theorem}{Theorem}[section]

\newtheorem{lemma}[theorem]{Lemma}

\theoremstyle{definition}
\newtheorem{definition}[theorem]{Definition}

\theoremstyle{remark}

\usepackage{amsmath,amsfonts,bm}

\def\eqref#1{equation~\ref{#1}}

\def\1{\bm{1}}

\DeclareMathAlphabet{\mathsfit}{\encodingdefault}{\sfdefault}{m}{sl}
\SetMathAlphabet{\mathsfit}{bold}{\encodingdefault}{\sfdefault}{bx}{n}

\def\sD{{\mathbb{D}}}

\usepackage[textsize=tiny]{todonotes}

\icmltitlerunning{Trustworthy Alignment of Retrieval-Augmented Large Language Models via Reinforcement Learning}

\newcommand{\method}{\textsc{Trustworthy-Alignment}}
\usepackage{enumitem}
\usepackage{dsfont}
\usepackage{multirow}
\usepackage{array}
\usepackage[most]{tcolorbox}
\newcommand{\textblue}[1]{\textcolor{blue}{#1}}
\newcommand{\textred}[1]{\textcolor{red}{#1}}

\begin{document}

\twocolumn[
\icmltitle{Trustworthy Alignment of Retrieval-Augmented Large Language Models \\
			via Reinforcement Learning}



\icmlsetsymbol{equal}{*}

\begin{icmlauthorlist}
\icmlauthor{Zongmeng Zhang}{ustc}
\icmlauthor{Yufeng Shi}{ustc}
\icmlauthor{Jinhua Zhu}{ustc}
\icmlauthor{Wengang Zhou}{ustc,iai}
\icmlauthor{Xiang Qi}{ant}
\icmlauthor{Peng Zhang}{ant}
\icmlauthor{Houqiang Li}{ustc,iai}
\end{icmlauthorlist}

\icmlaffiliation{ustc}{University of Science and Technology of China}
\icmlaffiliation{iai}{Institute of Artificial Intelligence, Hefei Comprehensive Nation Science Center}
\icmlaffiliation{ant}{Ant Group}

\icmlcorrespondingauthor{Wengang Zhou}{zhwg@ustc.edu.cn}
\icmlcorrespondingauthor{Houqiang Li}{lihq@ustc.edu.cn}

\icmlkeywords{Large Language Models, Alignment, Reinforcement Learning, Retrieval-Augmented Generation, Knowledge Conflicts, Trustworthy Machine Learning}

\vskip 0.3in
]



\printAffiliationsAndNotice{}  

\begin{abstract}
Trustworthiness is an essential prerequisite for the real-world application of large language models. In this paper, we focus on the trustworthiness of language models with respect to retrieval augmentation. 
Despite being supported with external evidence, retrieval-augmented generation still suffers from hallucinations, one primary cause of which is the conflict between contextual and parametric knowledge. 
We deem that retrieval-augmented language models have the inherent capabilities of supplying response according to both contextual and parametric knowledge. 
Inspired by aligning language models with human preference, we take the first step towards aligning retrieval-augmented language models to a status where it responds relying merely on the external evidence and disregards the interference of parametric knowledge. 
Specifically, we propose a reinforcement learning based algorithm \method{}, theoretically and experimentally demonstrating large language models' capability of reaching a trustworthy status without explicit supervision on how to respond. 
Our work highlights the potential of large language models on exploring its intrinsic abilities by its own and expands the application scenarios of alignment from fulfilling human preference to creating trustworthy agents. Our code is available at \url{https://github.com/zmzhang2000/trustworthy-alignment}.
\end{abstract}

\section{Introduction}

Recent advancement in Large Language Models (LLMs) has significantly boosted the application of Artificial Intelligence (AI) systems across a wide range of fields, such as dialogue assistants \cite{OpenAI-2023-GPT4} and autonomous agents~\cite{Wang-2023-Surveya}. Researchers contend that one of the most powerful LLMs, \emph{i.e.,} GPT-4, could reasonably be viewed as a spark of Artificial General Intelligence (AGI) considering its capabilities in language, vision, mathematics, coding, \emph{etc}~\cite{Bubeck-2023-Sparks}.

Despite these remarkable capabilities, the lack of trustworthiness is still a challenging obstacle to the application of LLMs. 
As probabilistic models, LLMs occasionally generate output which is plausible but deviating from source input~\cite{Adlakha-2023-Evaluating} or factual knowledge~\cite{Min-2023-FActScore,Muhlgay-2023-Generating,Li-2023-HaluEval}. 
This phenomenon is referred to as \emph{hallucination} \cite{Maynez-2020-Faithfulness,Ji-2023-Survey}, which poses tangible risks to real-world application scenarios with low tolerance of error such as medical diagnosis \cite{Nori-2023-Capabilities} and legal consultation \cite{Cui-2023-ChatLaw,Nguyen-2023-Brief}.

A common and straightforward approach to mitigate hallucinations in LLMs is to leverage non-parametric knowledge source through \emph{Retrieval-Augmented Generation} (RAG)~\cite{Lewis-2020-RetrievalAugmented,Shuster-2021-Retrieval}. 
Specifically, RAG retrieves supporting evidence relevant to the source input from an external knowledge base through a retriever, and generates the output conditioning on both the source input and retrieved evidence. The retrieve-and-generate paradigm proves to significantly enhance the factuality and reliability of LLMs \cite{Lewis-2020-RetrievalAugmented,Borgeaud-2022-Improving,Lazaridou-2022-Internetaugmented}.

However, incorporating retrieval into language generation encounters a new problem of \emph{knowledge conflicts} \cite{Longpre-2021-EntityBased} that still impedes the progress towards trustworthy language models. 
LLMs commonly learn from massive unsupervised data during pre-training and store the learned knowledge in their parameters. 
Conflict emerges as contextual knowledge in the retrieved evidence contrasting with the inner knowledge stored in LLMs' parameters, which triggers ``hallucinations'' if an LLM holds a strong belief in its parametric knowledge and produces response inconsistent with the provided context. 
This phenomenon is unexpected as the performance of RAG will be restricted by hallucinations if the retrieved evidence is not accurate and helpful enough \cite{Karpukhin-2020-Dense,Shi-2023-REPLUG,Ren-2023-Investigating,Mallen-2023-When}. 
We focus our research on the retrieval-augmented generation setting where an LLM is deemed trustworthy if it answers questions based on contextual rather than parametric knowledge.

A wealth of studies \cite{Longpre-2021-EntityBased,Chen-2022-Rich,Neeman-2023-DisentQA,Zhou-2023-Contextfaithful} has been dedicated to this problem owing to its obstructive impact on trustworthy LLMs. 
These efforts have discovered LLMs' dilemma between the parameters and retrieved evidence. 
In other words, LLMs may have capabilities of responding according to both contextual and parametric knowledge, and struggles to determine which to rely on. 
The inherent capabilities conforms to the foundation of a recently emerged technology in the field of AI: \emph{alignment} \cite{Ji-2023-AI}, which aims at encouraging a instruction-following language model~\cite{Wei-2022-FINETUNED} to behave in line with human intentions and values \cite{Leike-2018-Scalable} by reinforcement learning from human feedback (RLHF) \cite{Ouyang-2022-Training}. 
It is worth noting that alignment only ``unlocks'' capabilities LLMs already had, but does not teach new capabilities to the pre-trained and instruction-tuned models~\cite{Lowe-2022-Aligning}.

Inspired by the success of aligning LLMs with human intentions and values, we take the first step towards aligning retrieval-augmented LLM to a trustworthy status where it supplies responses in accordance with merely the retrieved evidence and ignores the parametric knowledge. 
We hypothesize that LLMs have tendencies of generating response in line with both contextual and parametric knowledge. 
Based on this, we address the problem of knowledge conflicts by aligning retrieval-augmented LLMs to trustworthy status via reinforcement learning (\textbf{\method}). In particular, we elaborate our work by answering the following research questions:
\begin{enumerate}[label=(${RQ}_{{\arabic*}}$),leftmargin=35pt]
	\itemsep-0.4em 
	\item \label{Q1}
	\emph{Does the hypothesis hold that retrieval-augmented LLMs have inherent capabilities of responding according to both contextual and parametric knowledge?} (\cref{sec:preliminary_study})
	\item \label{Q2}
	\emph{How can we design an algorithm that enhances the trustworthiness of retrieval-augmented LLMs by aligning them to a trustworthy status via reinforcement learning?} (\cref{sec:trustworthy_alignment})
	\item \label{Q3}
	\emph{How does the proposed algorithm perform and what are the side effects of aligning retrieval-augmented LLMs to trustworthy status?} (\cref{sec:experiments})
\end{enumerate}

We demonstrate that when \emph{merely providing what to achieve}, retrieval-augmented LLMs are capable of aligning itself to a status where it furnishes responses only relying on the contextual knowledge \emph{without explicit supervision on how to respond}. Our work highlights the potential of LLMs on exploring its intrinsic capabilities by its own and expands the application scenarios of alignment from fulfilling human intention and values to creating trustworthy agents.

\section{Preliminary Study}
\label{sec:preliminary_study}

In this section, we first provide the background of \method{} by formulating the task of retrieval-augmented generation, and discuss the problem of knowledge conflicts. Subsequently, we introduce answer probability, the proxy metric for model tendency analysis. Finally, we present the model tendency under different human-curated prompts, answering the prerequisite question of our work: \ref{Q1} \emph{does the hypothesis hold that retrieval-augmented LLMs have inherent capabilities of responding according to both contextual and parametric knowledge?}

\subsection{Retrieval-Augmented Generation}
 
A language model is a parametric probabilistic model $\pi_\theta$ conditioned on the input text. This paper focuses on the question answering (QA) task where the input text is a question $q$. The generation process is then formulated as sampling a response $y$ from the language model:
\begin{equation}
	\label{eqn:llm}
	y \sim \pi_\theta(y|q) = \prod_{t=0}^{T-1} \pi_\theta(y_t|q, y_{<t}),
\end{equation}
where both $x$ and $y$ are sequences of tokens, $t$ is the token index, and $T$ is the token number of the response.  

Under the retrieval-augmented setting, we have an additional context $c$ as a supporting evidence for answering the question, which is typically retrieved from a large knowledge base, \textit{e.g.}, Wikipedia. Following \citet{Brown-2020-Language}, we formulate the retrieval-augmented generation task by reformatting the question $q$ and context $c$ as a natural language instruction using a prompting function $f(q, c)$:
\begin{equation}
	y \sim \pi_\theta(y|q, c) = \prod_{t=0}^{T-1} \pi_\theta(y_t|f(q, c), y_{<t}).
\end{equation}

\subsection{Knowledge Conflicts}
Ideally, a retrieval-augmented LLM responds faithfully in line with the knowledge from $c$. However, the reality falls short of expectations since we have no idea whether the contextual knowledge in $c$ or the parametric knowledge in $\theta$ contributes more to LLM's decision. The \emph{knowledge conflict} \cite{Longpre-2021-EntityBased} emerges when contextual knowledge diverges from the parametric knowledge.

We simulate knowledge conflicts by constructing counterfactual data following the entity substitution framework in \citet{Longpre-2021-EntityBased}. 
Specifically, we have an original triplet $(q, c^o, y^o)$, where $y^o$ is the original right answer of question $q$ and is included in the context $c^o$ so that it can be inferred from $q$ and $c^o$. 
For each sample, we detect the corresponding entity of $y^o$ and replace it with a another entity $y^s$ of the same type. 
Meanwhile, $c^o$ is transformed to a new context $c^s$ by replacing all of the answer entity mentions. 
We expect the language model to supply a response including the substituted answer $y^s$ given the counterfactual context $c^s$.

\subsection{Tendency Analysis}
\label{sec:tendency_analysis}

\begin{figure}[t]
	\vskip 0.2in
	\begin{center}
		\centerline{\includegraphics[width=\columnwidth]{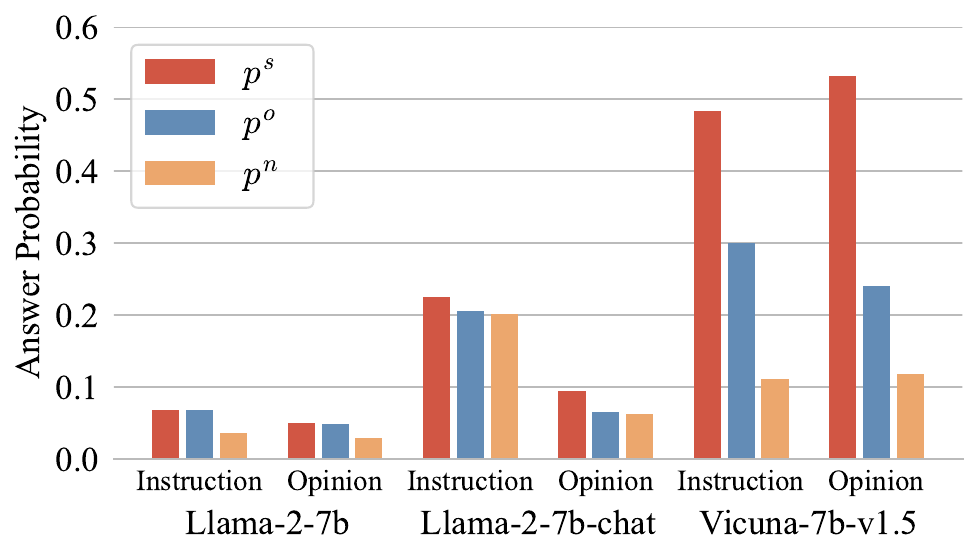}}
		\vskip -0.1in
		\caption{Average answer probabilities of the substituted answer $p^s$, original answer $p^o$ and none response $p^n$ for Llama-2-7b, Llama-2-7b-chat \cite{Touvron-2023-Llama} and Vicuna-7b-v1.5 \cite{Zheng-2023-Judging} with instruction-based and opinion-based prompting strategies. Generating tokens other than option codes is also feasible so that the sum of three probabilities is less than 1. Details of prompts are listed in \cref{app:prompts}.}
		\label{fig:tendency-analysis}
	\end{center}
	\vskip -0.2in
\end{figure}

To validate the prerequisite of our work that retrieval-augmented LLMs have tendencies of delivering response according to both the contextual and parametric knowledge, we conduct tendency analysis on multiple prevalent pre-trained LLMs. For a specific question $q$ and the counterfactual context $c^s$, we construct three candidate responses: the substituted answer $y^s$, the original answer $y^o$ and a none response $y^n$ instantiated as ``None of the above''. Then we format the question answering problem to a multi-choice task where LLM selects the best response from $y^s, y^o$ and $y^n$ given $q$ and $c^s$. We propose \emph{answer probability} as a proxy for measuring whether the model relies more on contextual or parametric knowledge, computed by
\begin{equation}
	\pi_\theta(\text{[OPTION]}|f_{\text{multi-choice}}(q, c^s, y^s, y^o, y^n)),
\end{equation}
where $\text{[OPTION]}$ is the token of option code, \textit{i.e.}, [A], [B] or [C] in out setting. We denote the probabilities of $y^s$, $y^o$ and $y^n$ corresponding to the options as $p^\text{s}$, $p^\text{o}$, $p^\text{n}$, respectively. 

We visualize the tendency of three pre-trained LLMs with two types of prompting strategies on the test set of Natural Questions \cite{Fisch-2019-MRQA} in \cref{fig:tendency-analysis}. With both the instruction-based and opinion-based prompts, Llama-2-7b and Llama-2-7b-chat give almost the same probabilities for the substituted and original answer. Vicuna-7b-v1.5 shows maximum sum of the three probabilities, indicating strongest understanding of the multi-choice task. Even so, Vicuna-7b-v1.5 struggles to decisively choose the substituted answer and still has an likelihood around 0.25 to select the original answer. Hence we draw the conclusion that \emph{the hypothesis of LLMs having capabilities of responding according to both the contextual and parametric knowledge holds true}.

\section{Trustworthy Alignment}
\label{sec:trustworthy_alignment}

The inherent capabilities of responding according to the contextual and parametric knowledge afford the opportunity to control the LLMs' intention in the scenario of knowledge conflict. In this section, we address \ref{Q2} by \emph{formulating the LLM as a Markov Decision Process (MDP), introducing \method{} and providing theoretical analysis for the proposed algorithm}.

\subsection{LLM as an MDP}
\label{sec:llm_as_an_mdp}

Typically, an MDP $\mathcal{M}\triangleq\left\langle \mathcal{S}, \mathcal{A}, r, \mathbb{P}, \rho_0, \gamma \right\rangle$ consists of the state space $\mathcal{S}$, the action space $\mathcal{A}$, a reward function $r$, the transition probability $\mathbb{P}$, the initial state distribution $\rho_0$ and a discount factor $\gamma$. The generation process of an LLM is naturally an MDP where each state $s \in \mathcal{S}$ is the context comprising prompt tokens $\tilde{a}_{0:m}$ and previously generated tokens $a_{0:t-1}$. The initial state $\tilde{s}$ is the sequence of prompt tokens $\tilde{a}_{0:m}$ sampled from the distribution $\rho_0$.\footnote{As all of the subsequent generation steps condition on the prompt tokens, we omit $\tilde{a}_{0:m}$ and denote state $s_t$ only with the generated tokens as $a_{0:t-1}$ for clarity.} Action space $\mathcal{A}$ is a finite vocabulary $\mathcal{V} = \{1, \ldots, |\mathcal{V}|\}$. Given the current state $s_t = a_{0:t-1}$, the LLM selects an action $a_{t}$ from $\mathcal{A}$, identically samples a token from the vocabulary $\mathcal{V}$ with the distribution $\pi_\theta$ in an auto-regressive manner formulated as $a_t \sim \pi_\theta(a_t|s_t) = \pi_\theta(a_t|a_{0:t-1})$. The transition in LLM is deterministic as the next state $s_{t+1}$ is definitely transformed into $a_{0:t}$ after the action $a_t$ is selected. Afterward, a deterministic reward $r_t$ is given by the reward function: $r_t = r(s_t, a_t)$. A state is regarded as terminated when a special end-of-text token is selected or the response reaches the max length. We define the state action value function $Q^{\pi_\theta}(s, a) \triangleq \mathbb{E}_{\mathbb{P}, \pi_\theta} \left[\sum^{T-1}_{t=0}\gamma^t r^t | s_0 = s, a_0 = a\right]$, which represents the expected return following policy $\pi_\theta$ after taking action $a$ at the state $s$. The state value function is defined as $V^{\pi_\theta} = \mathbb{E}_{a \sim \pi_\theta}\left[Q^{\pi_\theta}(s, a)\right]$, \emph{i.e.,} the expected return following $\pi_\theta$ from state $s$. The primary objective of optimizing the LLM via reinforcement learning is to maximize the expectation of discounted cumulative reward:
\begin{equation}
	J(\pi_\theta) = \mathbb{E}_{\tilde{s} \sim \rho_0} \mathbb{E}_{a_{0:T-1} \sim \pi_\theta} \left[\sum_{t=0}^{T-1} \gamma^t r_t \right].
\end{equation}

\subsection{Objective of Trustworthy-Alignment}
We aim at aligning the language model to the trustworthy status where it supplies responses in accordance with merely the retrieved evidence and ignores the parametric knowledge. To achieve this goal, we propose a simple yet effective reinforcement learning (RL) algorithm which rewards the model's reliance on contextual knowledge and penalizes its reliance on parametric knowledge. Specifically, given the question $q$ and substituted context $c^s$, we reward $\pi_\theta$ as long as the counterfactual answer $y^s$ appears in the terminated response and penalize it if the original answer $y^o$ appears or neither answers are included, formulated as
\begin{equation}
	\label{eqn:reward_trust}
	r_t^{\text{trust}} =
	\begin{cases}
		+\alpha_0 & \text{if } s_{t+1} \text{ ends} \wedge y^s \subseteq s_{t+1} \wedge y^o \nsubseteq s_{t+1} \\
		-\alpha_1 & \text{if } s_{t+1} \text{ ends} \wedge y^s \nsubseteq s_{t+1} \wedge y^o \nsubseteq s_{t+1} \\
		-\alpha_0 & \text{if } s_{t+1} \text{ ends} \wedge y^o \subseteq s_{t+1} \\
		0 & \text{otherwise} \\
	\end{cases},
\end{equation}
where $\alpha_0$ and $\alpha_1$ control absolute values of reward and penalty. Note that only the last generated token is rewarded so that this reward is on the sentence level.

Indispensably, a Kullback–Leibler (KL) divergence term is added to the reward following preceding works \cite{Jaques-2019-Way,Stiennon-2020-Learning} to penalize the policy's deviation from the original supervised fine-tuned model $\pi^\text{SFT}$, formulated as
\begin{equation}
	r_t^{\text{KL}} = -\alpha_2 \log \left[ \frac{\pi_\theta(a_t|s_t)}{\pi^{\text{SFT}}(a_t|s_t)} \right],
\end{equation}
where $\alpha_2$ is a coefficient controlling the training trade-off between a model performing specific task and a model more like the original one. 

Even with the regularization of KL divergence, LLMs are prone to falling into local optima in the training procedure of RL. For instance, pre-trained LLM may continually output the same repeated pattern when a short repetitive string like ``\textbackslash n\textbackslash n\textbackslash n'' shows up in the prompt, which will be reinforced and overwhelm the cumulative $r_t^{\text{trust}}$ in the subsequent training procedure by virtue of its low accumulative KL penalty. To avoid this, we propose an additional collapse penalty to persuade the LLM from collapsed mode:
\begin{equation}
	r_t^{\text{col}} = g(s_{t+1}) =
	\begin{cases}
		-\alpha_3 & \text{if } s_{t+1} \text{ ends and collapses}\\
		0 & \text{otherwise} \\
	\end{cases},
\end{equation}
where $\alpha_3$ controls the absolute value of $r_t^{\text{col}}$ and the collapse detection algorithm is described in \cref{app:implementation}.

Overall, the final reward function is formulated as
\begin{equation}
	\label{eqn:reward}
	r(s_t, a_t) = r_t^{\text{trust}} + r_t^{\text{col}} + r_t^{\text{KL}}.
\end{equation}
We train the policy with PPO algorithm \cite{Schulman-2017-Proximal}. We propose an annealing strategy to trade-off between exploration and exploitation in PPO training of LLMs, which is detailed in \cref{app:annealing_exploration} due to space constraints. The algorithm pipeline of \method{} is described in \cref{app:implementation}.

\subsection{Theoretical Analysis}
In this subsection, we theoretically analyze the proposed \method, ensuring the policy's convergence to the trustworthy status. Specifically, we demonstate that the optimal policy for the total reward $r(s_t, a_t)$ is consistent with the optimal policy for the task-specific reward $r_t^{\text{trust}}$, while $r_t^{\text{KL}}$ and $r_t^{\text{col}}$ do not affect the realization of the alignment. Actually, the purpose of designing these regular terms is to accelerate the model's learning procedure in practice \cite{Ng-1999-Policy}.

Formally, we define an original MDP as $\mathcal{M}\triangleq\left\langle \mathcal{S}, \mathcal{A}, r, \mathbb{P}, \rho_0, \gamma \right\rangle$, where $\mathcal{S}, \mathcal{A}, \mathbb{P}, \rho_0, \gamma$ share the same definitions as in \cref{sec:llm_as_an_mdp}. Both $\mathcal{S}$ and $\mathcal{A}$ are finite sets, and $r$ is defined as $r(s_t,a_t, s_{t+1}) = r^\text{trust}_t$.\footnote{Note that $r(s_t, a_t, s_{t+1}) = r(s_t, a_t)$, since $s_{t+1}$ can be uniquely determined by $s_t$ and $a_t$ in the context of language models.} A transformed MDP is defined based on the original MDP $\mathcal{M}$ as $\mathcal{M}'\triangleq\left\langle \mathcal{S}, \mathcal{A}, r', \mathbb{P}, \rho_0, \gamma \right\rangle$, which differs from $M$ in that the reward function is defined as $r'(s_t, a_t, s_{t+1}) = r^\text{trust}_t + r^\text{col}_t + r^\text{KL}_t$. The reward functions of both MDPs are bounded. We will show that the optimal ${\pi^*}'$ obtained on $\mathcal{M}'$ can be equivalent to the optimal $\pi^*$ obtained on $\mathcal{M}$. Our analysis approach is similar to the Potential-based Reward Shaping theory \cite{Ng-1999-Policy}. Firstly, we introduce the definition of the state potential function and provide a lemma.
\begin{definition}
	For any $s \in \mathcal{S}$, define its potential function as follows:
	\begin{equation}
		\Psi(s) = g(s) - \alpha_2 \log{P_{\pi_\theta}(s)} + \alpha_2 \log{P_{\pi^\text{SFT}}(s)},
	\end{equation}
\end{definition}
where $P_\pi(s) = \prod\pi_\theta(a_t|a_{0:t-1})$ indicates the probability of state $s$ generated from $\pi$. It can be known that $\Psi(s_{t + 1}) - \Psi(s_t) = r_t' - r_t = r^\text{col}_t + r^\text{KL}_t$, while the reward function $r^\text{trust}_t$ cannot be decomposed into a subtraction of some potential function, as discussed in detail in \cref{app:proofs}. Using the potential function, we can prove the following lemma about order-preserving properties.
\begin{lemma}(Order Preserving)
	\label{lemma:order}
	There exists a potential function $\Psi(\cdot)$ such that $\Psi(s_{t + 1}) - \Psi(s_t) = r_t' - r_t$. Given state $s \in \mathcal{S}$ and policy $\tilde{\pi}$, for any actions $a_1, a_2 \in \mathcal{A}$, when $\gamma = 1$, if  $Q^{\tilde{\pi}}_{\mathcal{M}'}(s, a_1) > Q^{\tilde{\pi}}_{\mathcal{M}'}(s, a_2)$, then $Q^{\tilde{\pi}}_{\mathcal{M}}(s, a_1) > Q^{\tilde{\pi}}_{\mathcal{M}}(s, a_2)$. Otherwise, assume $Q^{\tilde{\pi}}_{\mathcal{M}}$ is continuous function with respect to $\gamma$, when $\gamma$ is sufficiently close to $1$, if  $Q^{\tilde{\pi}}_{\mathcal{M}'}(s, a_1) > Q^{\tilde{\pi}}_{\mathcal{M}'}(s, a_2)$, then $Q^{\tilde{\pi}}_{\mathcal{M}}(s, a_1) > Q^{\tilde{\pi}}_{\mathcal{M}}(s, a_2)$.
\end{lemma}
The proof can be found in \cref{app:proofs}. This lemma indicates that when $\gamma$ is large enough, for the same policy, the relative numerical relationship of the Q-function of the transformed MDP $\mathcal{M}'$ is consistent with the original MDP $\mathcal{M}$. Therefore, we can obtain the following consistency and convergence theorem.

\begin{theorem} (Consistency and Convergence)
	Assume $Q^{\tilde{\pi}}_{\mathcal{M}}$ is a continuous function with respect to $\gamma$, when $\gamma$ is sufficiently close to $1$, the optimal policy obtained in MDP $\mathcal{M}'$  is also the optimal policy in MDP $\mathcal{M}$. The policy $\pi'$ obtained from the training of the RL algorithm has the following error bounds with the optimal policy, where $\pi^*$ is the optimal policy on $\mathcal{M}$, $\epsilon = \left\|  \mathcal{T}Q^{\pi'}_{\mathcal{M}} - Q^{\pi'}_{\mathcal{M}} \right\|_\infty$, $\mathcal{T}$ is the Bellman Optimal Operator:
	\begin{equation}
		\left\|  V^{\pi'}_{\mathcal{M}} - V^{\pi^*}_{\mathcal{M}} \right\|_\infty \leq \frac{2\gamma\epsilon}{1 - \gamma}.
	\end{equation}
\end{theorem}
The proof is detailed in \cref{app:proofs}. According to the theorem, we guarantee that LLMs trained with \method{} converge to the trustworthy status.

\begin{table*}[ht]
	\vskip -0.08in
	\centering
	\caption{Performance comparison between \method{} (Ours), supervised fine-tuning (SFT) and two prompting methods (Instruction and Opinion). Open-source models are trained on Natural Questions and evaluated on all three datasets. $\uparrow$ indicates that higher is better, while $\downarrow$ means lower is better. Results of best method for each model are highlighted in bold.}
	\vskip 0.15in
	\centering
	\small
	\newcolumntype{C}{>{\centering\arraybackslash}p{1.66cm}}
	\begin{tabular}{clCCCCCC}
		\toprule
		\multirow{2.6}{*}{Model} & \multirow{2.6}{*}{Method} & \multicolumn{2}{c}{\textsc{Natural Questions}} & \multicolumn{2}{c}{\textsc{ConflictQA-PopQA}} & \multicolumn{2}{c}{\textsc{ConflictQA-StrategyQA}} \\
		\cmidrule(lr){3-4} \cmidrule(lr){5-6} \cmidrule(lr){7-8} 
		& & EM (\%) $\uparrow$ & $M_R$ (\%) $\downarrow$ & Acc (\%) $\uparrow$ & $M_R$ (\%) $\downarrow$ & Acc (\%) $\uparrow$ & $M_R$ (\%) $\downarrow$ \\
		
		\midrule
		\multirow{2}{*}{ChatGPT} & Instruction & 47.73 & 45.23 & 85.61 & 9.09 & 84.29 & 21.06 \\
		& Opinion & 76.34 & 16.41 & 88.41 & 6.92 & 85.25 & 20.33 \\
		
		\midrule
		\multirow{2}{*}{GPT-4} & Instruction & 77.41 & 76.17 & 75.35 & 19.40 & 72.86 & 34.63 \\
		& Opinion & 89.71 & 84.14 & 86.19 & 11.43 & 83.63 & 20.65 \\
		
		\midrule
		\multirow{4}{*}{Llama-2-7b} & Instruction & 10.16 & 83.41 & 71.43 & 27.10 & 59.96 & 41.11 \\
		& Opinion & 20.32 & 66.97 & 68.67 & 26.82 & 43.22 & 50.98 \\
		& SFT & \textbf{92.34} & \textbf{1.58} & 76.19 & \textbf{13.30} & 67.38 & 39.76 \\
		& Ours & 82.93 & 1.65 & \textbf{81.19} & 18.20 & \textbf{69.17} & \textbf{37.32} \\
		
		\midrule
		\multirow{4}{*}{Vicuna-7b-v1.5} & Instruction & 48.80 & 39.57 & 79.15 & 10.25 & 66.96 & 29.44 \\
		& Opinion & 50.80 & 31.31 & 79.16 & \textbf{8.56} & 69.76 & 25.90 \\
		& SFT & \textbf{93.14} & \textbf{1.09} & \textbf{90.74} & 9.54 & 71.73 & 29.97 \\
		& Ours & 91.22 & 2.00 & 66.14 & 12.24 & \textbf{77.06} & \textbf{24.19} \\
		
		\midrule
		\multirow{4}{*}{Llama-2-7b-chat} & Instruction & 46.00 & 40.00 & 74.06 & 16.76 & 73.53 & 35.19 \\
		& Opinion & 62.17 & 24.54 & 74.31 & \textbf{10.59} & \textbf{75.07} & \textbf{31.73} \\
		& SFT & 92.42 & 1.24 & \textbf{85.31} & 15.40 & 73.70 & 38.35 \\
		& Ours & \textbf{94.92} & \textbf{0.70} & 66.14 & 12.24 & 64.87 & 31.79 \\
		
		\bottomrule
	\end{tabular}
	\label{tab:main_results}
	\vskip -0.15in
\end{table*}

\section{Experiments}
\label{sec:experiments}

In this section, we conduct experiments and analyze \method{} to answer \ref{Q3} \emph{how does the proposed algorithm perform and what are the side effects of aligning retrieval-augmented LLMs to trustworthy status}. We first detail the datasets, metrics and compared baselines. Then we analyze the performance, generalization capability, alignment tax and human evaluation of \method. Finally, we investigate the impact of individual components of \method{} by carrying out ablation studies on annealing exploration, scaling law and critic model.

\subsection{Setup}

\subsubsection{Datasets}

Following the setting of prior works \cite{Longpre-2021-EntityBased, Chen-2022-Rich}, we take the open-domain question answering task as a showcase of knowledge conflicts in the retrieval-augmented language model generation. We train \method{} on the MRQA Workshop Shared Task's version\footnote{There is a significant amount of duplicate data in the training set. We deduplicate the dataset for unbiased training.} \cite{Fisch-2019-MRQA} of \textbf{Natural Questions} \cite{Kwiatkowski-2019-Natural}. To be specific, we create the counterfactual version of Natural Questions by replacing the answer and all of its mentions in the context with another same type entity from the same dataset, following the substitution framework in \citet{Longpre-2021-EntityBased}. The original and substituted version of training set is used for training while only substituted version of dev set is used for evaluation.

Despite that the substituting framework is straightforward, \citet{Xie-2023-Adaptive} argues that evidence synthesized by simply substituting entities in the context is too incoherent to convince LLMs. And thus they construct a more coherent and convincing dataset \textbf{ConflictQA} on top of an entity-centric QA dataset \textbf{PopQA} \cite{Mallen-2023-When} and a multi-step reasoning dataset \textbf{StrategyQA} \cite{Geva-2021-Did} by prompting LLMs to directly generate supporting evidence that contradicts the parametric memory. We evaluate the generalization capability of \method{} by testing the performance of LLMs trained on Natural Questions on ConflictQA generated from GPT-4.

Besides, to evaluate the alignment tax, we test performance change of the aligned model on \textbf{Open LLM Leaderboard}\footnote{\url{https://huggingface.co/spaces/HuggingFaceH4/open_llm_leaderboard}} tasks, including grade-school science \textbf{AI2 Reasoning Challenge} \cite{Clark-2018-Think}, commonsense inference \textbf{HellaSwag} \cite{Zellers-2019-HellaSwag}, multitask language understanding \textbf{MMLU} \cite{Hendrycks-2021-Measuring}, truthful generation \textbf{TruthfulQA} \cite{Lin-2022-TruthfulQA}, commonsense reasoning \textbf{Winogrande} \cite{Sakaguchi-2021-WinoGrande} and multi-step mathematical reasoning \textbf{GSM8k} \cite{Cobbe-2021-Training}.

\subsubsection{Metrics}
We evaluate the response of LLM given the question and substituted evidence. For free-form QA dataset Natural Questions, we report the average \emph{best subspan exact match} (EM) indicating if the answer entity appears in the response. Samples in ConflictQA is transformed into multiple-choice QA format following \citet{Xie-2023-Adaptive} and thus we directly assess the \emph{accuracy} (Acc) for this dataset.

Additionally, we adopt memorization ratio \cite{Longpre-2021-EntityBased} $M_R = \frac{p_o}{p_o + p_s}$ to measure the LLM's capability of forgetting parametric knowledge, where $p_o$ and $p_s$ are the fractions of samples that model predicts the original answer and substituted answer, respectively. Note that $p_s$ is equivalent to average EM and Acc and an LLM is deemed as in trustworthy status if owning a small memorization ratio. Following \citet{Longpre-2021-EntityBased}, we narrow the evaluation on Natural Questions by only preserving samples that LLM can correctly respond the original answer in a closed-book setting, meaning that the knowledge is originally stored in its parameters. Nevertheless, we do not perform this filter on ConflictQA as multi-choice QA task is profoundly affected by the order of options \cite{Wang-2018-Position,Ko-2020-Look,Wang-2023-PandaLM,Zheng-2023-Judging}, which renders the filtering unreliable.

\subsubsection{Baselines}
\begin{itemize}
	\itemsep-0.4em 
	\item \textbf{Closed-source Models.} We experiment with two closed-source LLMs, ChatGPT (gpt-3.5-turbo-1106) and GPT-4 (gpt-4-1106-preview)\footnote{Experiments on the GPT series models are conducted in January, 2024.}, to present the state-of-the art performance under the zero-shot setting.
	\item \textbf{Open-source Models.} We conduct experiments on Llama-2-7b and its fine-tuned version Llama-2-7b-chat that has been optimized for dialogue use cases \cite{Touvron-2023-Llama}. Besides, we also evaluate Vicuna-7b-v1.5 \cite{Zheng-2023-Judging} which is fine-tuned from Llama-2-7b with ChatGPT conversation data.
	\item \textbf{Prompt Engineering.} We evaluate the zero-shot performance of these closed-source and open-source LLMs with prompt engineering. An instruction-based and an opinion-based prompting strategy are adopted for all models. All prompts are listed in \cref{app:prompts}.
	\item \textbf{Supervised Fine-tuning.} We perform supervised fine-tuning on open-source LLMs, which directly teaches the LLM how to respond by minimizing the cross entropy between the ground truth answer and output of LLM. Prompt template is fixed to instruction-based form in the supervised fine-tuning and \method.
\end{itemize}

\subsection{Analyses}
\label{sec:analysis}

\begin{figure*}[ht]
	\vskip 0.2in
	\begin{center}
		\centerline{\includegraphics[width=\textwidth]{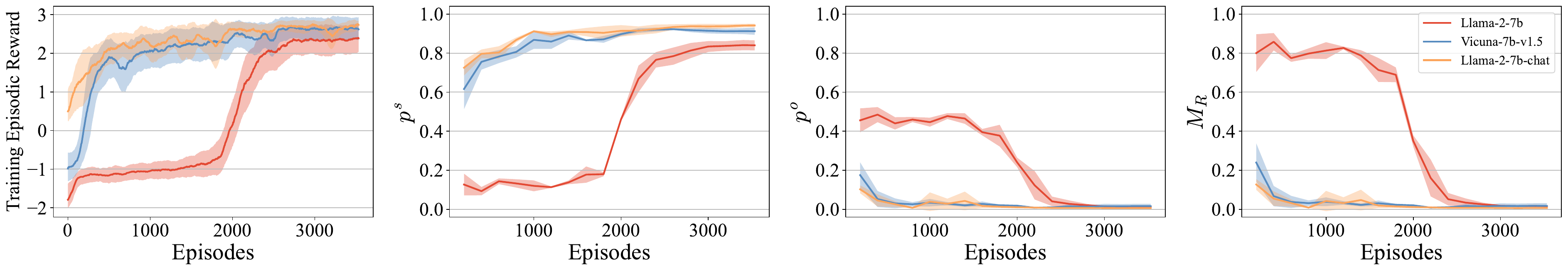}}
		\caption{Curves of training episodic reward, $p^s$, $p^o$ and $M_R$ on held-out evaluation dataset in the training procedure of \method{} on Natural Questions. The solid curves represent the mean value and the shaded areas correspond to the standard deviations of 3 runs with different seeds. The curve of episodic reward is smoothed by averaging within a sliding window. For a clear presentation of task performance, the KL penalty is not included in the curve of episodic reward but visualized alone in \cref{app:supp_results}.}
		\label{fig:training_curves}
	\end{center}
	\vskip -0.2in
\end{figure*}

\begin{table*}[ht]
	\centering
	\caption{The performance change of LLMs after \method{} (Ours) and supervised fine-tuning (SFT) on Open LLM Leaderboard tasks. Performance discrepancies between fine-tuned and pre-trained models are enclosed in parentheses.}
	\vskip 0.15in
	\centering
	\small
	\newcolumntype{C}{>{\centering\arraybackslash}p{1.78cm}}
	\begin{tabular}{clCCCCCC}
		\toprule
		Model & Method & ARC $\uparrow$ & HellaSwag $\uparrow$ & MMLU $\uparrow$ & TruthfulQA $\uparrow$ & Winogrande $\uparrow$ & GSM8k $\uparrow$ \\
		
		\midrule
		\multirow{3}{*}{Llama-2-7b} & Base & 46.25 & 76.01 & 41.77 & 38.96 & 69.06 & 13.19 \\
		& SFT & 47.95 (+1.70) & 76.40 (+0.39) & 41.82 (+0.05) & 43.53 (+4.57) & 69.30 (+0.24) & 10.31 (-2.88) \\
		& Ours & 45.90 (-0.35) & 75.60 (-0.41) & 39.26 (-2.51) & 36.95 (-2.01) & 68.59 (-0.47) & 11.68 (-1.51) \\
		
		\midrule
		\multirow{3}{*}{Vicuna-7b-v1.5} & Base & 45.73 & 73.76 & 48.75 & 50.37 & 69.46 & 19.18 \\
		& SFT & 44.88 (-0.85) & 74.62 (+0.86) & 46.84 (-1.91) & 48.82 (-1.55) & 68.43 (-1.03) & 12.96 (-6.22) \\
		& Ours & 45.73 (+0.00) & 72.94 (-0.82) & 48.62 (-0.13) & 50.00 (-0.37) & 69.61 (+0.15) & 18.42 (-0.76) \\
		
		\midrule
		\multirow{3}{*}{Llama-2-7b-chat} & Base & 44.37 & 75.45 & 45.34 & 44.59 & 66.22 & 21.68 \\
		& SFT & 47.70 (+3.33) & 75.02 (-0.43) & 46.17 (+0.83) & 45.43 (+0.84) & 68.03 (+1.81) & 14.10 (-7.58) \\
		& Ours & 47.87 (+3.50) & 75.39 (-0.06) & 48.33 (+2.99) & 45.17 (+0.58) & 65.82 (-0.40) & 17.13 (-4.55) \\
		
		\bottomrule
	\end{tabular}
	\label{tab:alignment_tax}
	\vskip -0.1in
\end{table*}

\paragraph{The proposed alignment algorithm significantly improves the trustworthiness of retrieval-augmented LLMs.}
We present the performance of \method{} and compared baselines on Natural Questions in \cref{tab:main_results}. From the results, we observe that with training-free prompt engineering approach, GPT-4 serves as the most effective model with respect to EM. The hand-crafted opinion-based prompt shows better EM and lower $M_R$ than instruction-based prompt on almost all of these model. The opinion-based prompt formulates the contextual information as somebody's opinion and thus LLM believes what is provided and ignores what is true in its parametric knowledge. However, GPT-4 memorizes more often than other models. We analyze the output of GPT-4, finding that it always presents not only facts in the opinion but also knowledge stored in its parameters, and thus shows highest memorization ratio. Examples are shown in \cref{app:examples}.

With \method{} and supervised fine-tuning, all of three open-source models achieves significantly higher EM and lower $M_R$. Notably, even without explicit supervision on how to respond, \method{} achieves comparable performance to supervised fine-tuning and even surpasses it when using the stronger model Llama-2-7b-chat. Performance of \method{} on Llama-2-7b and Vicuna-7b-v1.5 is inferior to supervised fine-tuning due to weaker instruction-following capability of base models, which will be discussed in the following subsection.

\paragraph{Inherent capacity and adequate exploration are indispensable in aligning LLMs to trustworthy status.}
To delve more profoundly into the training process of \method, we visualize curves of training episodic reward, as well as $p^s$, $p^o$ and $M_R$ on held-out evaluation dataset. As shown in \cref{fig:training_curves}, Llama-2-7b-chat and Vicuna-7b-v1.5 enjoy stable training where the training episodic reward and $p^s$ climb up while $p^o$ and $M_R$ decline gradually. Llama-2-7b behaves differently from the other two models. In the early training phase, Llama-2-7b is the weakest model and holds the lowest training episodic reward. Continued poor performance persists due to a lack of exploration in high-reward areas. As the LLM gradually reaches the area where substituted answers appear, \textit{i.e.}, near 2000 episodes in \cref{fig:training_curves}, the performance displays a steep climb to a status with high training episodic reward and low evaluation $M_R$. This phenomenon emphasizes the demand of inherent trustworthy status in pre-trained LLMs and exploration in the training process for \method. Generally, the performance of all three models converges as the training carrying on, demonstrating the stability of the \method{} algorithm with inherent trustworthy status in pre-trained LLMs.

\paragraph{The trustworthy aligned model generalizes well to out-of-domain datasets.}
We evaluate the LLMs aligned with Natural Questions data on two additional datasets ConflictQA-PopQA and ConflictQA-StrategyQA, whose questions and contexts have never been seen and even the task is transformed from free-form QA to multi-choice QA. As shown in \cref{tab:main_results}, \method{} improves the accuracy of Llama-2-7b and Vicuna-7b-v1.5 on ConflictQA-PopQA and ConflictQA-StrategyQA over prompt engineering approaches. Meanwhile, \method{} surpasses SFT with respect to the metric $M_R$ by a large margin for the Llama-2-7b and Vicuna-7b-v1.5 model on ConflictQA-StrategyQA, indicating its more significant suppressing effect on parametric memory. This verifies the favorable generalization capability of \method{} to out-of-domain datasets.

\paragraph{Aligning LLMs via RL preserves language styles of original pre-trained models.} A by-product of aligning LLMs via RL is that language styles of original pre-trained models are preserved, while SFT changes them totally. We showcase responses of LLMs aligned via \method{} and SFT in \cref{app:examples}. As the results show, the Llama-2-13b-chat aligned with \method{} gives answers with complete sentences. By contrast, Llama-2-13b-chat aligned with SFT provides barely answers, even without periods. This is a reasonable phenomenon as the demonstration data used in RL is actually collected from the LLM in the training process, which keeps the same language style with the pre-trained model. The data used in SFT is labeled by human whose style distribution may be distinct from the original model. This phenomenon is also an evidence that the trustworthy status is an inherent ability of LLMs rather than assigned by external supervision data. Additionally, we provide human evaluation results on language style and fluency of \method{} in \cref{sec:human_evaluation} to further demonstrates that aligning via RL indeed helps.

\paragraph{Aligning LLMs to trustworthy status via RL enjoys preferential alignment tax treatment.}

\citet{Ouyang-2022-Training} declare that training LLM with PPO suffers from ``alignment tax'', \textit{i.e.}, model's performance drops on several public NLP datasets after aligning the model. We evaluate trustworthy aligned LLMs on the Open LLM Leaderboard task set. As the results in \cref{tab:alignment_tax} show, the performance of aligned Vicuna-7b-v1.5 on these six datasets drops less than 1\% while the supervised fine-tuning reduces the performance more. Unexpectedly, \method{} enhances the Llama-2-7b-chat's performance on some of these out-of-domain datasets, like ARC, MMLU and TruthfulQA. A possible explanation is that, \method{} encourages Llama-2-7b-chat to concentrate more on the information in prompt, which is helpful for these improved datasets. We also find that whether it is \method{} or supervised fine-tuning, both decline noticeably on GSM8k, though \method{} surpasses supervised fine-tuning with all of three base models. This may be due to the limited error tolerance of mathematical reasoning task and thus more future works are needed to further mitigate the alignment tax. 

\begin{figure}[t]
	\vskip 0.2in
	\begin{center}
		\centerline{\includegraphics[width=0.7\columnwidth]{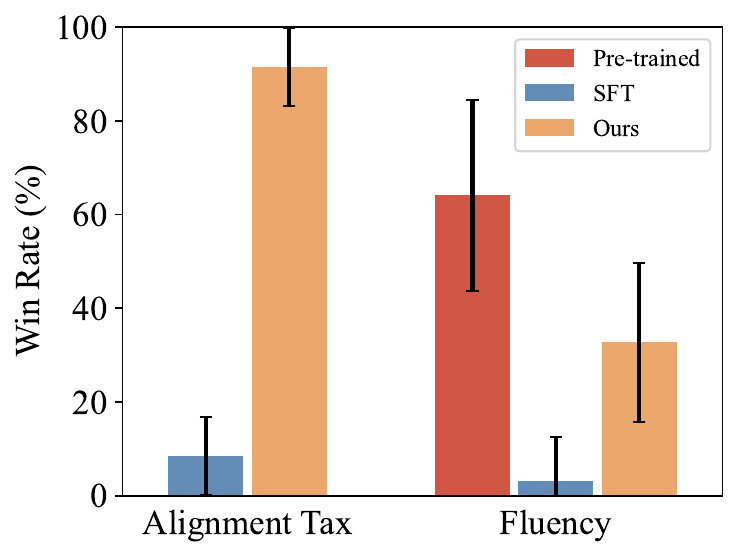}}
		\vskip -0.1in
		\caption{Win rate of \method{} (Ours) and supervised fine-tuning (SFT) in human evaluation for alignment tax and impact on fluency. We present the mean win rate along with error bars that represent the standard deviations.}
		\label{fig:human-evaluation}
	\end{center}
	\vskip -0.2in
\end{figure}

\subsection{Human Evaluation}
\label{sec:human_evaluation}

Despite that we do not intend to align LLMs with human preference via the proposed algorithm, human evaluation is an indispensable evaluation protocol for language generation. Thus, we conduct human evaluations for alignment tax and fluency of \method{}. Specifically, we randomly select question-context pairs from the test set of Natural Questions, generate responses using pre-trained, supervised-finetuned and trustworthy aligned Llama-2-7b-chat models for each pair. Human evaluators are asked to judge whether the response from supervised-finetuned or trustworthy aligned model is closer to the pre-trained model in terms of language style. Naturally, more similar language style with pre-trained model indicates lower alignment tax. Besides, we ask human evaluators to select their most preferred response out of these three models for measurement of fluency.

We collect feedback on 40 question-context pairs from 33 human evaluators, and compute the mean values and standard deviations of win rate for each model. According to the results shown in \cref{fig:human-evaluation}, the responses of \method{} are significantly more analogous to the pre-trained model and more fluent than SFT model. These results are expected since the RL algorithm in \method{} reinforces some of the intrinsic outputs of pre-trained LLM, while SFT forces the pre-trained LLM to follow the style of provided demonstration data.

\subsection{Ablation Studies}

\begin{table}[t]
	\centering
	\caption{Results of ablation studies on the annealing exploration strategy, scaling law and critic model. All experiments are conducted on Natural Questions dataset.}
	\vskip 0.15in
	\centering
	\small
	\begin{tabular}{lcc}
		\toprule
		Variant & EM (\%) $\uparrow$ & $M_R$ (\%) $\downarrow$ \\
		\midrule
		wo. Annealing & 96.12 & 1.10 \\
		\midrule
		1.3B & 24.04 & 14.29 \\
		13B & 94.39 & 1.27 \\
		\midrule
		Random Critic & 93.58 & 0.71 \\
		Larger Critic & 93.32 & 1.85 \\
		\midrule
		Origin & 94.92 & 0.70 \\
		\bottomrule
	\end{tabular}
	\label{tab:ablation_study}
	\vskip -0.1in
\end{table}

We carry out ablation studies on the annealing exploration strategy, scaling law and critic model in \method{} to verify their impacts. Based on an aligned Llama-2-7b-chat with chat-opt-350m-reward\footnote{\url{https://huggingface.co/AdamG012/chat-opt-350m-reward-deepspeed}.} as pre-trained critic model and annealing exploration strategy (\textbf{Origin}), we consider the following variants:
\begin{itemize}
	\itemsep-0.4em 
	\item We disable the annealing exploration strategy by training with a fixed temperature 1.0 (\textbf{wo. Annealing}).
	\item We investigate the scaling law of \method{} by training with different actor model size, including a 1.3B model chat-opt-1.3b\footnote{\url{https://huggingface.co/AdamG012/chat-opt-1.3b-sft-deepspeed}.} (\textbf{1.3B}) and a 13B model Llama-2-13b-chat (\textbf{13B})
	\item We replace critic model in \method{} with a random initialized chat-opt-350m-reward (\textbf{Random Critic}) and a pre-trained Llama-2-7b-chat without head layer (\textbf{Larger Critic}) 
\end{itemize}

As shown in \cref{tab:ablation_study}, ``wo. Annealing'' obtains higher EM and $M_R$ than baseline, which indicates a much higher $p^o$. These results are consistent with the expectation that exploration plays a crucial role in the whole aligning procedure discussed in \cref{sec:analysis}. Without large temperature in the early stage of annealing exploration, LLMs generate less experience excluding the parametric knowledge, and thus is less penalized for memorizing.

Likewise, 1.3B model exhibits a much inferior performance resulting from lack of exploration on output including the contextual knowledge. The larger 13B model does not show markedly better performance than the original 7B model. We speculate that 7B model is enough for approaching the upper limit of our tasks so that the verification of scaling law for larger models may requires stronger benchmarks.

Results of ``Random Critic'' and ``Larger Critic'' indicate that neither the pre-training nor scaling up critic models in alignment via reinforcement learning exerts substantial influence on the performance, which is consistent with one of the conclusions in \citet{Li-2023-ReMax}.

\section{Related Work}

\subsection{Hallucination}

Hallucination in LLMs is defined as generating content that is nonsensical or contradict to the source input and typically divided into extrinsic and intrinsic hallucinations \cite{Ji-2023-Survey}. An LLM may present responses that cannot be verified from the source or generated content \cite{Maynez-2020-Faithfulness}, which is referred to as extrinsic hallucination. Retrieval-Augmented Generation \cite{Lewis-2020-RetrievalAugmented} attempts to tackle this type of hallucination by storing knowledge in an extra knowledge base, retrieve and relying on related knowledge during generation. Despite that verifiability is significantly enhanced, the conflict between parametric and contextual knowledge \cite{Longpre-2021-EntityBased} triggers the intrinsic hallucination where LLMs generate outputs inconsistent with retrieved information \cite{Mallen-2023-When}. \citet{Longpre-2021-EntityBased} and \citet{Neeman-2023-DisentQA} construct counterfactual data and utilize supervised fine-tuning to compel LLMs to focus on the retrieved context. \citet{Chen-2022-Rich} expand the study of knowledge conflicts to multiple evidence setting and proposes a recalibration method encouraging LLMs to abstain from answering questions in conflict. Prompt engineering is also proved to be effective for improving faithfulness in knowledge conflict situations \cite{Zhou-2023-Contextfaithful,Si-2023-PROMPTING}. These dedicated efforts imply inherent capabilities of responding according to parametric and contextual knowledge in LLMs. Based on this, we take the first step towards encouraging LLMs to explore trustworthy status with reinforcement learning where only what to achieve is provided.

\subsection{Alignment}

Alignment was first introduced to create agents that behave in line with the intention of users \cite{Leike-2018-Scalable} and boomed in the era of LLMs due to the strong demand for AI assistants harmony with human values \cite{Stiennon-2020-Learning,Ouyang-2022-Training}. Alignment is achieved by reinforcement learning from human feedback (RLHF) \cite{Christiano-2017-Deep} that trains a reward model representing human preference and optimize an instruction-tuned model \cite{Wei-2022-FINETUNED} using the reward model. Beyond that, researchers extend the RLHF to reinforcement learning from AI feedback (RLAIF) \cite{Bai-2022-Constitutional} and self-alignment \cite{Sun-2023-PrincipleDriven} that learn with feedback generated from language models to reduce the dependence of human annotations. Moreover, safe RL framework \cite{Garcia-2015-Comprehensive} is integrated into alignment to mitigate the conflict between helpfulness and harmless of LLMs. A grand vision of aligning future superhuman models with weak models \cite{Burns-2023-WeaktoStrong} confirms the long-term value of alignment technology in the development of AGI. Different from existing works that mainly focus on aligning agents with human intention and values, we expand the application scenarios of alignment to creating trustworthy agents.

\section{Conclusion}

In this work, we take the first step towards expanding the application scenarios of alignment technology from fulfilling human intentions to creating trustworthy agents that overcome knowledge conflicts. Specifically, we elaborate our work by verifying the prerequisite hypothesis, design a reinforcement learning algorithm and comprehensively investigate the performance and side effects of the proposed algorithm. We theoretically and experimentally demonstrate LLMs' capability of reaching a trustworthy status without explicit supervision on how to respond. We hope our work will provide insight into future research in the theory and applications of alignment technology.

\section*{Acknowledgements}

This work is supported by National Key R\&D Program of China under Contract 2022ZD0119802, National Natural Science Foundation of China under Contract 623B2097 and the Youth Innovation Promotion Association CAS. It was supported by GPU cluster built by MCC Lab of Information Science and Technology Institution, USTC, and the Supercomputing Center of the USTC. This work was also supported by Ant Group Research Fund.

\section*{Impact Statement}

This paper presents work whose goal is to advance the field of 
Machine Learning. There are many potential societal consequences 
of our work, none which we feel must be specifically highlighted here.


\bibliography{strings, example_paper}

\begin{thebibliography}{61}
\providecommand{\natexlab}[1]{#1}
\providecommand{\url}[1]{\texttt{#1}}
\expandafter\ifx\csname urlstyle\endcsname\relax
  \providecommand{\doi}[1]{doi: #1}\else
  \providecommand{\doi}{doi: \begingroup \urlstyle{rm}\Url}\fi

\bibitem[Adlakha et~al.(2023)Adlakha, BehnamGhader, Lu, Meade, and
  Reddy]{Adlakha-2023-Evaluating}
Adlakha, V., BehnamGhader, P., Lu, X.~H., Meade, N., and Reddy, S.
\newblock Evaluating correctness and faithfulness of instruction-following
  models for question answering.
\newblock \emph{arXiv:2307.16877}, 2023.

\bibitem[Bai et~al.(2022)Bai, Kadavath, Kundu, Askell, Kernion, Jones, Chen,
  Goldie, Mirhoseini, McKinnon, Chen, Olsson, Olah, Hernandez, Drain, Ganguli,
  Li, {Tran-Johnson}, Perez, Kerr, Mueller, Ladish, Landau, Ndousse, Lukosuite,
  Lovitt, Sellitto, Elhage, Schiefer, Mercado, DasSarma, Lasenby, Larson,
  Ringer, Johnston, Kravec, Showk, Fort, Lanham, {Telleen-Lawton}, Conerly,
  Henighan, Hume, Bowman, {Hatfield-Dodds}, Mann, Amodei, Joseph, McCandlish,
  Brown, and Kaplan]{Bai-2022-Constitutional}
Bai, Y., Kadavath, S., Kundu, S., Askell, A., Kernion, J., Jones, A., Chen, A.,
  Goldie, A., Mirhoseini, A., McKinnon, C., Chen, C., Olsson, C., Olah, C.,
  Hernandez, D., Drain, D., Ganguli, D., Li, D., {Tran-Johnson}, E., Perez, E.,
  Kerr, J., Mueller, J., Ladish, J., Landau, J., Ndousse, K., Lukosuite, K.,
  Lovitt, L., Sellitto, M., Elhage, N., Schiefer, N., Mercado, N., DasSarma,
  N., Lasenby, R., Larson, R., Ringer, S., Johnston, S., Kravec, S., Showk,
  S.~E., Fort, S., Lanham, T., {Telleen-Lawton}, T., Conerly, T., Henighan, T.,
  Hume, T., Bowman, S.~R., {Hatfield-Dodds}, Z., Mann, B., Amodei, D., Joseph,
  N., McCandlish, S., Brown, T., and Kaplan, J.
\newblock Constitutional {AI}: Harmlessness from ai feedback.
\newblock \emph{arXiv:2212.08073}, 2022.

\bibitem[Borgeaud et~al.(2021)Borgeaud, Mensch, Hoffmann, Cai, Rutherford,
  Millican, van~den Driessche, Lespiau, Damoc, Clark, Casas, Guy, Menick, Ring,
  Hennigan, Huang, Maggiore, Jones, Cassirer, Brock, Paganini, Irving, Vinyals,
  Osindero, Simonyan, Rae, Elsen, and Sifre]{Borgeaud-2022-Improving}
Borgeaud, S., Mensch, A., Hoffmann, J., Cai, T., Rutherford, E., Millican, K.,
  van~den Driessche, G., Lespiau, J.-B., Damoc, B., Clark, A., Casas, D. d.~L.,
  Guy, A., Menick, J., Ring, R., Hennigan, T., Huang, S., Maggiore, L., Jones,
  C., Cassirer, A., Brock, A., Paganini, M., Irving, G., Vinyals, O., Osindero,
  S., Simonyan, K., Rae, J.~W., Elsen, E., and Sifre, L.
\newblock Improving language models by retrieving from trillions of tokens.
\newblock \emph{arXiv:2112.04426}, 2021.

\bibitem[Brown et~al.(2020)Brown, Mann, Ryder, Subbiah, Kaplan, Dhariwal,
  Neelakantan, Shyam, Sastry, Askell, Agarwal, {Herbert-Voss}, Krueger,
  Henighan, Child, Ramesh, Ziegler, Wu, Winter, Hesse, Chen, Sigler, Litwin,
  Gray, Chess, Clark, Berner, McCandlish, Radford, Sutskever, and
  Amodei]{Brown-2020-Language}
Brown, T.~B., Mann, B., Ryder, N., Subbiah, M., Kaplan, J., Dhariwal, P.,
  Neelakantan, A., Shyam, P., Sastry, G., Askell, A., Agarwal, S.,
  {Herbert-Voss}, A., Krueger, G., Henighan, T., Child, R., Ramesh, A.,
  Ziegler, D.~M., Wu, J., Winter, C., Hesse, C., Chen, M., Sigler, E., Litwin,
  M., Gray, S., Chess, B., Clark, J., Berner, C., McCandlish, S., Radford, A.,
  Sutskever, I., and Amodei, D.
\newblock Language models are few-shot learners.
\newblock \emph{arXiv:2005.14165}, 2020.

\bibitem[Bubeck et~al.(2023)Bubeck, Chandrasekaran, Eldan, Gehrke, Horvitz,
  Kamar, Lee, Lee, Li, Lundberg, Nori, Palangi, Ribeiro, and
  Zhang]{Bubeck-2023-Sparks}
Bubeck, S., Chandrasekaran, V., Eldan, R., Gehrke, J., Horvitz, E., Kamar, E.,
  Lee, P., Lee, Y.~T., Li, Y., Lundberg, S., Nori, H., Palangi, H., Ribeiro,
  M.~T., and Zhang, Y.
\newblock Sparks of artificial general intelligence: Early experiments with
  {GPT}-4.
\newblock \emph{arXiv:2303.12712}, 2023.

\bibitem[Burns et~al.(2023)Burns, Izmailov, Kirchner, Baker, Gao,
  Aschenbrenner, Chen, Ecoffet, Joglekar, Leike, Sutskever, and
  Wu]{Burns-2023-WeaktoStrong}
Burns, C., Izmailov, P., Kirchner, J.~H., Baker, B., Gao, L., Aschenbrenner,
  L., Chen, Y., Ecoffet, A., Joglekar, M., Leike, J., Sutskever, I., and Wu, J.
\newblock Weak-to-strong generalization: Eliciting strong capabilities with
  weak supervision.
\newblock \emph{arXiv:2312.09390}, 2023.

\bibitem[Chen et~al.(2022)Chen, Zhang, and Choi]{Chen-2022-Rich}
Chen, H.-T., Zhang, M., and Choi, E.
\newblock Rich knowledge sources bring complex knowledge conflicts:
  Recalibrating models to reflect conflicting evidence.
\newblock In \emph{Proceedings of the Conference on Empirical Methods in
  Natural Language Processing}, pp.\  2292--2307. Association for Computational
  Linguistics, 2022.

\bibitem[Christiano et~al.(2017)Christiano, Leike, Brown, Martic, Legg, and
  Amodei]{Christiano-2017-Deep}
Christiano, P., Leike, J., Brown, T.~B., Martic, M., Legg, S., and Amodei, D.
\newblock Deep reinforcement learning from human preferences.
\newblock \emph{arXiv:1706.03741}, 2017.

\bibitem[Clark et~al.(2018)Clark, Cowhey, Etzioni, Khot, Sabharwal, Schoenick,
  and Tafjord]{Clark-2018-Think}
Clark, P., Cowhey, I., Etzioni, O., Khot, T., Sabharwal, A., Schoenick, C., and
  Tafjord, O.
\newblock Think you have solved question answering? try {ARC}, the {AI2}
  reasoning challenge.
\newblock \emph{arXiv:1803.05457}, 2018.

\bibitem[Cobbe et~al.(2021)Cobbe, Kosaraju, Bavarian, Chen, Jun, Kaiser,
  Plappert, Tworek, Hilton, Nakano, Hesse, and Schulman]{Cobbe-2021-Training}
Cobbe, K., Kosaraju, V., Bavarian, M., Chen, M., Jun, H., Kaiser, L., Plappert,
  M., Tworek, J., Hilton, J., Nakano, R., Hesse, C., and Schulman, J.
\newblock Training verifiers to solve math word problems.
\newblock \emph{arXiv:2110.14168}, 2021.

\bibitem[Cui et~al.(2023)Cui, Li, Yan, Chen, and Yuan]{Cui-2023-ChatLaw}
Cui, J., Li, Z., Yan, Y., Chen, B., and Yuan, L.
\newblock Chatlaw: Open-source legal large language model with integrated
  external knowledge bases.
\newblock \emph{arXiv:2306.16092}, 2023.

\bibitem[Fisch et~al.(2019)Fisch, Talmor, Jia, Seo, Choi, and
  Chen]{Fisch-2019-MRQA}
Fisch, A., Talmor, A., Jia, R., Seo, M., Choi, E., and Chen, D.
\newblock {MRQA} 2019 shared task: Evaluating generalization in reading
  comprehension.
\newblock In \emph{Proceedings of the Workshop on Machine Reading for Question
  Answering}, pp.\  1--13. Association for Computational Linguistics, 2019.

\bibitem[Garc{\i}a \& Fernandez(2015)Garc{\i}a and
  Fernandez]{Garcia-2015-Comprehensive}
Garc{\i}a, J. and Fernandez, F.
\newblock A comprehensive survey on safe reinforcement learning.
\newblock \emph{Journal of Machine Learning Research}, 16\penalty0
  (1):\penalty0 1437--1480, 2015.

\bibitem[Geva et~al.(2021)Geva, Khashabi, Segal, Khot, Roth, and
  Berant]{Geva-2021-Did}
Geva, M., Khashabi, D., Segal, E., Khot, T., Roth, D., and Berant, J.
\newblock Did aristotle use a laptop? a question answering benchmark with
  implicit reasoning strategies.
\newblock \emph{Transactions of the Association for Computational Linguistics},
  9:\penalty0 346--361, 2021.

\bibitem[Hendrycks et~al.(2021)Hendrycks, Burns, Basart, Zou, Mazeika, Song,
  and Steinhardt]{Hendrycks-2021-Measuring}
Hendrycks, D., Burns, C., Basart, S., Zou, A., Mazeika, M., Song, D., and
  Steinhardt, J.
\newblock Measuring massive multitask language understanding.
\newblock In \emph{Proceedings of the International Conference on Learning
  Representations}, 2021.

\bibitem[Hu et~al.(2021)Hu, Shen, Wallis, {Allen-Zhu}, Li, Wang, Wang, and
  Chen]{Hu-2021-LoRA}
Hu, E.~J., Shen, Y., Wallis, P., {Allen-Zhu}, Z., Li, Y., Wang, S., Wang, L.,
  and Chen, W.
\newblock Lora: Low-rank adaptation of large language models.
\newblock \emph{arXiv:2106.09685}, 2021.

\bibitem[Jaques et~al.(2019)Jaques, Ghandeharioun, Shen, Ferguson, Lapedriza,
  Jones, Gu, and Picard]{Jaques-2019-Way}
Jaques, N., Ghandeharioun, A., Shen, J.~H., Ferguson, C., Lapedriza, A., Jones,
  N., Gu, S., and Picard, R.
\newblock Way off-policy batch deep reinforcement learning of implicit human
  preferences in dialog.
\newblock \emph{arXiv:1907.00456}, 2019.

\bibitem[Ji et~al.(2023{\natexlab{a}})Ji, Qiu, Chen, Zhang, Lou, Wang, Duan,
  He, Zhou, Zhang, Zeng, Ng, Dai, Pan, O'Gara, Lei, Xu, Tse, Fu, McAleer, Yang,
  Wang, Zhu, Guo, and Gao]{Ji-2023-AI}
Ji, J., Qiu, T., Chen, B., Zhang, B., Lou, H., Wang, K., Duan, Y., He, Z.,
  Zhou, J., Zhang, Z., Zeng, F., Ng, K.~Y., Dai, J., Pan, X., O'Gara, A., Lei,
  Y., Xu, H., Tse, B., Fu, J., McAleer, S., Yang, Y., Wang, Y., Zhu, S.-C.,
  Guo, Y., and Gao, W.
\newblock {AI} alignment: A comprehensive survey.
\newblock \emph{arXiv:2310.19852}, 2023{\natexlab{a}}.

\bibitem[Ji et~al.(2023{\natexlab{b}})Ji, Lee, Frieske, Yu, Su, Xu, Ishii,
  Bang, Madotto, and Fung]{Ji-2023-Survey}
Ji, Z., Lee, N., Frieske, R., Yu, T., Su, D., Xu, Y., Ishii, E., Bang, Y.~J.,
  Madotto, A., and Fung, P.
\newblock Survey of hallucination in natural language generation.
\newblock \emph{ACM Computing Surveys}, 55\penalty0 (12):\penalty0 1--38,
  2023{\natexlab{b}}.

\bibitem[Karpukhin et~al.(2020)Karpukhin, Oguz, Min, Lewis, Wu, Edunov, Chen,
  and Yih]{Karpukhin-2020-Dense}
Karpukhin, V., Oguz, B., Min, S., Lewis, P., Wu, L., Edunov, S., Chen, D., and
  Yih, W.-t.
\newblock Dense passage retrieval for open-domain question answering.
\newblock In \emph{Proceedings of the Conference on Empirical Methods in
  Natural Language Processing}, pp.\  6769--6781. Association for Computational
  Linguistics, 2020.

\bibitem[Ko et~al.(2020)Ko, Lee, Kim, Kim, and Kang]{Ko-2020-Look}
Ko, M., Lee, J., Kim, H., Kim, G., and Kang, J.
\newblock Look at the first sentence: Position bias in question answering.
\newblock In \emph{Proceedings of the Conference on Empirical Methods in
  Natural Language Processing}, pp.\  1109--1121. Association for Computational
  Linguistics, 2020.

\bibitem[Kwiatkowski et~al.(2019)Kwiatkowski, Palomaki, Redfield, Collins,
  Parikh, Alberti, Epstein, Polosukhin, Devlin, Lee, Toutanova, Jones, Kelcey,
  Chang, Dai, Uszkoreit, Le, and Petrov]{Kwiatkowski-2019-Natural}
Kwiatkowski, T., Palomaki, J., Redfield, O., Collins, M., Parikh, A., Alberti,
  C., Epstein, D., Polosukhin, I., Devlin, J., Lee, K., Toutanova, K., Jones,
  L., Kelcey, M., Chang, M.-W., Dai, A.~M., Uszkoreit, J., Le, Q., and Petrov,
  S.
\newblock {Natural Questions}: A benchmark for question answering research.
\newblock \emph{Transactions of the Association for Computational Linguistics},
  7:\penalty0 452--466, 2019.

\bibitem[Lazaridou et~al.(2022)Lazaridou, Gribovskaya, Stokowiec, and
  Grigorev]{Lazaridou-2022-Internetaugmented}
Lazaridou, A., Gribovskaya, E., Stokowiec, W., and Grigorev, N.
\newblock Internet-augmented language models through few-shot prompting for
  open-domain question answering.
\newblock \emph{arXiv:2203.05115}, 2022.

\bibitem[Leike et~al.(2018)Leike, Krueger, Everitt, Martic, Maini, and
  Legg]{Leike-2018-Scalable}
Leike, J., Krueger, D., Everitt, T., Martic, M., Maini, V., and Legg, S.
\newblock Scalable agent alignment via reward modeling: a research direction.
\newblock \emph{arXiv:1811.07871}, 2018.

\bibitem[Lewis et~al.(2020)Lewis, Perez, Piktus, Petroni, Karpukhin, Goyal,
  K{\"u}ttler, Lewis, Yih, Rockt{\"a}schel, Riedel, and
  Kiela]{Lewis-2020-RetrievalAugmented}
Lewis, P., Perez, E., Piktus, A., Petroni, F., Karpukhin, V., Goyal, N.,
  K{\"u}ttler, H., Lewis, M., Yih, W.-t., Rockt{\"a}schel, T., Riedel, S., and
  Kiela, D.
\newblock Retrieval-augmented generation for knowledge-intensive {NLP} tasks.
\newblock In \emph{Proceedings of the Advances in Neural Information Processing
  Systems}, 2020.

\bibitem[Li et~al.(2023{\natexlab{a}})Li, Cheng, Zhao, Nie, and
  Wen]{Li-2023-HaluEval}
Li, J., Cheng, X., Zhao, X., Nie, J.-Y., and Wen, J.-R.
\newblock Halueval: A large-scale hallucination evaluation benchmark for large
  language models.
\newblock In \emph{Proceedings of the Conference on Empirical Methods in
  Natural Language Processing}, pp.\  6449--6464. Association for Computational
  Linguistics, 2023{\natexlab{a}}.

\bibitem[Li et~al.(2023{\natexlab{b}})Li, Xu, Zhang, Yu, Sun, and
  Luo]{Li-2023-ReMax}
Li, Z., Xu, T., Zhang, Y., Yu, Y., Sun, R., and Luo, Z.-Q.
\newblock Remax: A simple, effective, and efficient reinforcement learning
  method for aligning large language models.
\newblock \emph{arXiv:2310.10505}, 2023{\natexlab{b}}.

\bibitem[Lin et~al.(2022)Lin, Hilton, and Evans]{Lin-2022-TruthfulQA}
Lin, S., Hilton, J., and Evans, O.
\newblock {TruthfulQA}: Measuring how models mimic human falsehoods.
\newblock In \emph{Proceedings of the Annual Meeting of the Association for
  Computational Linguistics}, pp.\  3214--3252. Association for Computational
  Linguistics, 2022.

\bibitem[Longpre et~al.(2021)Longpre, Perisetla, Chen, Ramesh, DuBois, and
  Singh]{Longpre-2021-EntityBased}
Longpre, S., Perisetla, K., Chen, A., Ramesh, N., DuBois, C., and Singh, S.
\newblock Entity-based knowledge conflicts in question answering.
\newblock In \emph{Proceedings of the Conference on Empirical Methods in
  Natural Language Processing}, pp.\  7052--7063. Association for Computational
  Linguistics, 2021.

\bibitem[Lowe \& Leike(2022)Lowe and Leike]{Lowe-2022-Aligning}
Lowe, R. and Leike, J.
\newblock Aligning language models to follow instructions.
\newblock https://openai.com/research/instruction-following, 2022.

\bibitem[Mallen et~al.(2023)Mallen, Asai, Zhong, Das, Khashabi, and
  Hajishirzi]{Mallen-2023-When}
Mallen, A., Asai, A., Zhong, V., Das, R., Khashabi, D., and Hajishirzi, H.
\newblock When not to trust language models: Investigating effectiveness of
  parametric and non-parametric memories.
\newblock In \emph{Proceedings of the Annual Meeting of the Association for
  Computational Linguistics}, pp.\  9802--9822. Association for Computational
  Linguistics, 2023.

\bibitem[Maynez et~al.(2020)Maynez, Narayan, Bohnet, and
  McDonald]{Maynez-2020-Faithfulness}
Maynez, J., Narayan, S., Bohnet, B., and McDonald, R.
\newblock On faithfulness and factuality in abstractive summarization.
\newblock In \emph{Proceedings of the Annual Meeting of the Association for
  Computational Linguistics}, pp.\  1906--1919. Association for Computational
  Linguistics, 2020.

\bibitem[Min et~al.(2023)Min, Krishna, Lyu, Lewis, Yih, Koh, Iyyer,
  Zettlemoyer, and Hajishirzi]{Min-2023-FActScore}
Min, S., Krishna, K., Lyu, X., Lewis, M., Yih, W.-t., Koh, P., Iyyer, M.,
  Zettlemoyer, L., and Hajishirzi, H.
\newblock {FActScore}: Fine-grained atomic evaluation of factual precision in
  long form text generation.
\newblock In \emph{Proceedings of the Conference on Empirical Methods in
  Natural Language Processing}, pp.\  12076--12100. Association for
  Computational Linguistics, 2023.

\bibitem[Mnih et~al.(2015)Mnih, Kavukcuoglu, Silver, Rusu, Veness, Bellemare,
  Graves, Riedmiller, Fidjeland, Ostrovski, Petersen, Beattie, Sadik,
  Antonoglou, King, Kumaran, Wierstra, Legg, and
  Hassabis]{Mnih-2015-Humanlevel}
Mnih, V., Kavukcuoglu, K., Silver, D., Rusu, A.~A., Veness, J., Bellemare,
  M.~G., Graves, A., Riedmiller, M., Fidjeland, A.~K., Ostrovski, G., Petersen,
  S., Beattie, C., Sadik, A., Antonoglou, I., King, H., Kumaran, D., Wierstra,
  D., Legg, S., and Hassabis, D.
\newblock Human-level control through deep reinforcement learning.
\newblock \emph{Nature}, 518\penalty0 (7540):\penalty0 529--533, 2015.

\bibitem[Muhlgay et~al.(2023)Muhlgay, Ram, Magar, Levine, Ratner, Belinkov,
  Abend, {Leyton-Brown}, Shashua, and Shoham]{Muhlgay-2023-Generating}
Muhlgay, D., Ram, O., Magar, I., Levine, Y., Ratner, N., Belinkov, Y., Abend,
  O., {Leyton-Brown}, K., Shashua, A., and Shoham, Y.
\newblock Generating benchmarks for factuality evaluation of language models.
\newblock \emph{arXiv:2307.06908}, 2023.

\bibitem[Neeman et~al.(2023)Neeman, Aharoni, Honovich, Choshen, Szpektor, and
  Abend]{Neeman-2023-DisentQA}
Neeman, E., Aharoni, R., Honovich, O., Choshen, L., Szpektor, I., and Abend, O.
\newblock {DisentQA}: Disentangling parametric and contextual knowledge with
  counterfactual question answering.
\newblock In \emph{Proceedings of the Annual Meeting of the Association for
  Computational Linguistics}, pp.\  10056--10070. Association for Computational
  Linguistics, 2023.

\bibitem[Ng et~al.(1999)Ng, Harada, and Russell]{Ng-1999-Policy}
Ng, A.~Y., Harada, D., and Russell, S.
\newblock Policy invariance under reward transformations: Theory and
  application to reward shaping.
\newblock In \emph{Proceedings of the International Conference on Machine
  Learning}, pp.\  278--287, 1999.

\bibitem[Nguyen(2023)]{Nguyen-2023-Brief}
Nguyen, H.-T.
\newblock A brief report on lawgpt 1.0: A virtual legal assistant based on
  {GPT-3}.
\newblock \emph{arXiv:2302.05729}, 2023.

\bibitem[Nori et~al.(2023)Nori, King, McKinney, Carignan, and
  Horvitz]{Nori-2023-Capabilities}
Nori, H., King, N., McKinney, S.~M., Carignan, D., and Horvitz, E.
\newblock Capabilities of {GPT-4} on medical challenge problems.
\newblock \emph{arXiv:2303.13375}, 2023.

\bibitem[OpenAI(2023)]{OpenAI-2023-GPT4}
OpenAI.
\newblock {GPT-4} technical report.
\newblock Technical Report arXiv:2303.08774, 2023.

\bibitem[Ouyang et~al.(2022)Ouyang, Wu, Jiang, Almeida, Wainwright, Mishkin,
  Zhang, Agarwal, Slama, Ray, Schulman, Hilton, Kelton, Miller, Simens, Askell,
  Welinder, Christiano, Leike, and Lowe]{Ouyang-2022-Training}
Ouyang, L., Wu, J., Jiang, X., Almeida, D., Wainwright, C.~L., Mishkin, P.,
  Zhang, C., Agarwal, S., Slama, K., Ray, A., Schulman, J., Hilton, J., Kelton,
  F., Miller, L., Simens, M., Askell, A., Welinder, P., Christiano, P., Leike,
  J., and Lowe, R.
\newblock Training language models to follow instructions with human feedback.
\newblock In \emph{Proceedings of the Advances in Neural Information Processing
  Systems}, 2022.

\bibitem[Ren et~al.(2023)Ren, Wang, Qu, Zhao, Liu, Tian, Wu, Wen, and
  Wang]{Ren-2023-Investigating}
Ren, R., Wang, Y., Qu, Y., Zhao, W.~X., Liu, J., Tian, H., Wu, H., Wen, J.-R.,
  and Wang, H.
\newblock Investigating the factual knowledge boundary of large language models
  with retrieval augmentation.
\newblock \emph{arXiv:2307.11019}, 2023.

\bibitem[Sakaguchi et~al.(2021)Sakaguchi, Bras, Bhagavatula, and
  Choi]{Sakaguchi-2021-WinoGrande}
Sakaguchi, K., Bras, R.~L., Bhagavatula, C., and Choi, Y.
\newblock {WinoGrande}: an adversarial winograd schema challenge at scale.
\newblock \emph{Communications of the ACM}, 64\penalty0 (9):\penalty0 99--106,
  2021.

\bibitem[Schulman et~al.(2017)Schulman, Wolski, Dhariwal, Radford, and
  Klimov]{Schulman-2017-Proximal}
Schulman, J., Wolski, F., Dhariwal, P., Radford, A., and Klimov, O.
\newblock Proximal policy optimization algorithms.
\newblock \emph{arXiv:1707.06347}, 2017.

\bibitem[Shi et~al.(2023)Shi, Min, Yasunaga, Seo, James, Lewis, Zettlemoyer,
  and Yih]{Shi-2023-REPLUG}
Shi, W., Min, S., Yasunaga, M., Seo, M., James, R., Lewis, M., Zettlemoyer, L.,
  and Yih, W.-t.
\newblock {REPLUG}: Retrieval-augmented black-box language models.
\newblock \emph{arXiv:2301.12652}, 2023.

\bibitem[Shuster et~al.(2021)Shuster, Poff, Chen, Kiela, and
  Weston]{Shuster-2021-Retrieval}
Shuster, K., Poff, S., Chen, M., Kiela, D., and Weston, J.
\newblock Retrieval augmentation reduces hallucination in conversation.
\newblock In \emph{Findings of the Association for Computational Linguistics:
  EMNLP 2021}, pp.\  3784--3803. Association for Computational Linguistics,
  2021.

\bibitem[Si et~al.(2023)Si, Gan, Yang, Wang, Wang, {Boyd-Graber}, and
  Wang]{Si-2023-PROMPTING}
Si, C., Gan, Z., Yang, Z., Wang, S., Wang, J., {Boyd-Graber}, J., and Wang, L.
\newblock Prompting {GPT-3} to be reliable.
\newblock In \emph{Proceedings of the International Conference on Learning
  Representations}, 2023.

\bibitem[Stiennon et~al.(2020)Stiennon, Ouyang, Wu, Ziegler, Lowe, Voss,
  Radford, Amodei, and Christiano]{Stiennon-2020-Learning}
Stiennon, N., Ouyang, L., Wu, J., Ziegler, D., Lowe, R., Voss, C., Radford, A.,
  Amodei, D., and Christiano, P.~F.
\newblock Learning to summarize with human feedback.
\newblock In \emph{Advances in Neural Information Processing Systems},
  volume~33, pp.\  3008--3021. Curran Associates, Inc., 2020.

\bibitem[Sun et~al.(2023)Sun, Shen, Zhou, Zhang, Chen, Cox, Yang, and
  Gan]{Sun-2023-PrincipleDriven}
Sun, Z., Shen, Y., Zhou, Q., Zhang, H., Chen, Z., Cox, D., Yang, Y., and Gan,
  C.
\newblock Principle-driven self-alignment of language models from scratch with
  minimal human supervision.
\newblock \emph{arXiv:2305.03047}, 2023.

\bibitem[Sutton \& Barto(2018)Sutton and Barto]{Sutton-2018-Reinforcement}
Sutton, R.~S. and Barto, A.~G.
\newblock \emph{Reinforcement learning: an introduction}.
\newblock The MIT Press, 2018.
\newblock ISBN 978-0-262-03924-6.

\bibitem[Touvron et~al.(2023)Touvron, Martin, Stone, Albert, Almahairi, Babaei,
  Bashlykov, Batra, Bhargava, Bhosale, Bikel, Blecher, Ferrer, Chen, Cucurull,
  Esiobu, Fernandes, Fu, Fu, Fuller, Gao, Goswami, Goyal, Hartshorn, Hosseini,
  Hou, Inan, Kardas, Kerkez, Khabsa, Kloumann, Korenev, Koura, Lachaux, Lavril,
  Lee, Liskovich, Lu, Mao, Martinet, Mihaylov, Mishra, Molybog, Nie, Poulton,
  Reizenstein, Rungta, Saladi, Schelten, Silva, Smith, Subramanian, Tan, Tang,
  Taylor, Williams, Kuan, Xu, Yan, Zarov, Zhang, Fan, Kambadur, Narang,
  Rodriguez, Stojnic, Edunov, and Scialom]{Touvron-2023-Llama}
Touvron, H., Martin, L., Stone, K., Albert, P., Almahairi, A., Babaei, Y.,
  Bashlykov, N., Batra, S., Bhargava, P., Bhosale, S., Bikel, D., Blecher, L.,
  Ferrer, C.~C., Chen, M., Cucurull, G., Esiobu, D., Fernandes, J., Fu, J., Fu,
  W., Fuller, B., Gao, C., Goswami, V., Goyal, N., Hartshorn, A., Hosseini, S.,
  Hou, R., Inan, H., Kardas, M., Kerkez, V., Khabsa, M., Kloumann, I., Korenev,
  A., Koura, P.~S., Lachaux, M.-A., Lavril, T., Lee, J., Liskovich, D., Lu, Y.,
  Mao, Y., Martinet, X., Mihaylov, T., Mishra, P., Molybog, I., Nie, Y.,
  Poulton, A., Reizenstein, J., Rungta, R., Saladi, K., Schelten, A., Silva,
  R., Smith, E.~M., Subramanian, R., Tan, X.~E., Tang, B., Taylor, R.,
  Williams, A., Kuan, J.~X., Xu, P., Yan, Z., Zarov, I., Zhang, Y., Fan, A.,
  Kambadur, M., Narang, S., Rodriguez, A., Stojnic, R., Edunov, S., and
  Scialom, T.
\newblock Llama 2: Open foundation and fine-tuned chat models.
\newblock \emph{arXiv:2307.09288}, 2023.

\bibitem[Wang et~al.(2023{\natexlab{a}})Wang, Ma, Feng, Zhang, Yang, Zhang,
  Chen, Tang, Chen, Lin, Zhao, Wei, and Wen]{Wang-2023-Surveya}
Wang, L., Ma, C., Feng, X., Zhang, Z., Yang, H., Zhang, J., Chen, Z., Tang, J.,
  Chen, X., Lin, Y., Zhao, W.~X., Wei, Z., and Wen, J.-R.
\newblock A survey on large language model based autonomous agents.
\newblock \emph{arXiv:2308.11432}, 2023{\natexlab{a}}.

\bibitem[Wang et~al.(2018)Wang, Golbandi, Bendersky, Metzler, and
  Najork]{Wang-2018-Position}
Wang, X., Golbandi, N., Bendersky, M., Metzler, D., and Najork, M.
\newblock Position bias estimation for unbiased learning to rank in personal
  search.
\newblock In \emph{Proceedings of the ACM International Conference on Web
  Search and Data Mining}, pp.\  610--618. ACM, 2018.

\bibitem[Wang et~al.(2023{\natexlab{b}})Wang, Yu, Zeng, Yang, Wang, Chen,
  Jiang, Xie, Wang, Xie, Ye, Zhang, and Zhang]{Wang-2023-PandaLM}
Wang, Y., Yu, Z., Zeng, Z., Yang, L., Wang, C., Chen, H., Jiang, C., Xie, R.,
  Wang, J., Xie, X., Ye, W., Zhang, S., and Zhang, Y.
\newblock {PandaLM}: An automatic evaluation benchmark for llm instruction
  tuning optimization.
\newblock \emph{arXiv:2306.05087}, 2023{\natexlab{b}}.

\bibitem[Wei et~al.(2022)Wei, Bosma, Zhao, Guu, Yu, Lester, Du, Dai, and
  Le]{Wei-2022-FINETUNED}
Wei, J., Bosma, M., Zhao, V.~Y., Guu, K., Yu, A.~W., Lester, B., Du, N., Dai,
  A.~M., and Le, Q.~V.
\newblock Finetuned language models are zero-shot learners.
\newblock In \emph{Proceedings of the International Conference on Learning
  Representations}, 2022.

\bibitem[Williams et~al.(1993)Williams, Baird, Laboratory, and
  Afb]{Williams-1993-Tight}
Williams, R.~J., Baird, L.~C., Laboratory, W., and Afb, W.-P.
\newblock Tight performance bounds on greedy policies based on imperfect value
  functions.
\newblock Technical report, 1993.

\bibitem[Xie et~al.(2023)Xie, Zhang, Chen, Lou, and Su]{Xie-2023-Adaptive}
Xie, J., Zhang, K., Chen, J., Lou, R., and Su, Y.
\newblock Adaptive chameleon or stubborn sloth: Revealing the behavior of large
  language models in knowledge conflicts.
\newblock \emph{arXiv:2305.13300}, 2023.

\bibitem[Yao et~al.(2023)Yao, Aminabadi, Ruwase, Rajbhandari, Wu, Awan, Rasley,
  Zhang, Li, Holmes, Zhou, Wyatt, Smith, Kurilenko, Qin, Tanaka, Che, Song, and
  He]{Yao-2023-DeepSpeedChat}
Yao, Z., Aminabadi, R.~Y., Ruwase, O., Rajbhandari, S., Wu, X., Awan, A.~A.,
  Rasley, J., Zhang, M., Li, C., Holmes, C., Zhou, Z., Wyatt, M., Smith, M.,
  Kurilenko, L., Qin, H., Tanaka, M., Che, S., Song, S.~L., and He, Y.
\newblock Deepspeed-chat: Easy, fast and affordable rlhf training of
  chatgpt-like models at all scales.
\newblock \emph{arXiv:2308.01320}, 2023.

\bibitem[Zellers et~al.(2019)Zellers, Holtzman, Bisk, Farhadi, and
  Choi]{Zellers-2019-HellaSwag}
Zellers, R., Holtzman, A., Bisk, Y., Farhadi, A., and Choi, Y.
\newblock Hellaswag: Can a machine really finish your sentence?
\newblock In \emph{Proceedings of the Annual Meeting of the Association for
  Computational Linguistics}, pp.\  4791--4800. Association for Computational
  Linguistics, 2019.

\bibitem[Zheng et~al.(2023)Zheng, Chiang, Sheng, Zhuang, Wu, Zhuang, Lin, Li,
  Li, Xing, Zhang, Gonzalez, and Stoica]{Zheng-2023-Judging}
Zheng, L., Chiang, W.-L., Sheng, Y., Zhuang, S., Wu, Z., Zhuang, Y., Lin, Z.,
  Li, Z., Li, D., Xing, E.~P., Zhang, H., Gonzalez, J.~E., and Stoica, I.
\newblock Judging {LLM}-as-a-judge with {MT}-bench and chatbot arena.
\newblock In \emph{Proceedings of the Advances in Neural Information Processing
  Systems}, 2023.

\bibitem[Zhou et~al.(2023)Zhou, Zhang, Poon, and
  Chen]{Zhou-2023-Contextfaithful}
Zhou, W., Zhang, S., Poon, H., and Chen, M.
\newblock Context-faithful prompting for large language models.
\newblock In \emph{Findings of the Association for Computational Linguistics},
  pp.\  14544--14556. Association for Computational Linguistics, 2023.

\end{thebibliography}
\bibliographystyle{icml2024}

\newpage
\appendix
\onecolumn
\section{Proofs}
\label{app:proofs}
\begin{definition}
	\label{app:def}
	For any $s \in \mathcal{S}$, define its potential function as follows, where $P_\pi(s)$ indicates the probability of generating state $s$ using $\pi$:
	\begin{equation}
		\Psi(s) = g(s) - \alpha_2 \log{P_{\pi_\theta}(s)} + \alpha_2 \log{P_{\pi_\text{SFT}}(s)}.
	\end{equation}
\end{definition}
When $s_{t+1}$ is a collapsed string, it is clear that $s_t$ cannot be a collapsed string, so there is $g(s_{t+1}) - g(s_t) = g(s_{t+1}) - 0 = r^\text{col}_t$. Otherwise, neither $s_{t+1}$ nor $s_t$ is a collapsed string, so there is $g(s_{t+1}) - g(s_t) = 0 - 0 = r^\text{col}_t$. According to \cref{app:def}, we have:
\begin{equation}
	\label{eqn:potential}
	\begin{aligned}
		&\Psi(s_{t+1}) - \Psi(s_{t}) = g(s_{t+1}) - g(s_t) - \alpha_2 \left(\log{P_{\pi_\theta}(s_{t+1})} - \log{P_{\pi_\theta}(s_{t})} - \left( \log{P_{\pi_\text{SFT}}(s_{t+1})} - \log{P_{\pi_\text{SFT}}(s_{t})} \right) \right)\\
		&= r^\text{col}_t - \alpha_2 \left( \log{\prod^{t}_{k=0}\pi_\theta(a_k|a_{0:k-1})} - \log{\prod^{t-1}_{k=0}\pi_\theta(a_k|a_{0:k-1})} - \left( \log{\prod^{t}_{k=0}\pi_\text{SFT}(a_k|a_{0:k-1})} - \log{\prod^{t-1}_{k=0}\pi_\text{SFT}(a_k|a_{0:k-1})} \right) \right) \\
		&= r^\text{col}_t - \alpha_2 \left( \sum^{t}_{k=0}{\log{\pi_\theta(a_k|a_{0:k-1})}} - \sum^{t-1}_{k=0}{\log{\pi_\theta(a_k|a_{0:k-1})}} - \left( \sum^{t}_{k=0}{\log{\pi_\text{SFT}(a_k|a_{0:k-1})}} - \sum^{t-1}_{k=0}{\log{\pi_\text{SFT}(a_k|a_{0:k-1})}} \right) \right)\\
		&=  r^\text{col}_t - \alpha_2 \left( \log{\pi_\theta(a_t|a_{0:t-1})} - \log{\pi_\text{SFT}(a_t|a_{0:t-1})} \right) = r^\text{col}_t - \alpha_2 \left( \log{\pi_\theta(a_t|s_t)} - \log{\pi_\text{SFT}(a_t|s_t)} \right) \\
		&= r^\text{col}_t - \alpha_2\log{\frac{\pi_\theta(a_t|s_t)}{\pi_\text{SFT}(a_t|s_t)}} = r^\text{col}_t + r^\text{KL}_t = r'(s_t, a_t, s_{t + 1}) - r(s_t, a_t, s_{t + 1}).
	\end{aligned}
\end{equation}
Equation (\ref{eqn:potential}) shows that the rewards $r^\text{col}_t$ and $r^\text{KL}_t$ can be expressed as the difference between two neighboring state potential functions. Note that $r^\text{trust}_t$ cannot be expressed as the difference between the potential functions of two neighboring states, since $r^\text{trust}_t$ depends not only on the state $s_{t+1}$, but also on $y^s$ and $y^o$. $y^s$ and $y^o$ can, however, be viewed as feedback following an MDP input action (i.e., an MDP input action $a_t$ outputs $S_{t+1},y^s,y^o,r^\text{trust}_t$). However, this feedback is not given directly, but rather in an indirect form as a reward $r^\text{trust}_t$. Since $r^\text{trust}_t$ cannot be expressed as a difference of potential functions, based on to Potential-based Reward Shaping theory \cite{Ng-1999-Policy}, it is known that that there exists a preferred optimal policy for the original MDP $M$ (the preference policy guided by rewards in MDP $M$ is the policy that produces trustworthy responses) and that there is no case where any policy without any preference is optimal.

\begin{lemma}(Order Preserving)
	\label{app:lemma:order}
	There exists a potential function $\Psi(\cdot)$ such that $\Psi(s_{t + 1}) - \Psi(s_t) = r_t' - r_t$. Given state $s \in \mathcal{S}$ and policy $\tilde{\pi}$, for any actions $a_1, a_2 \in \mathcal{A}$, when $\gamma = 1$, if  $Q^{\tilde{\pi}}_{\mathcal{M}'}(s, a_1) > Q^{\tilde{\pi}}_{\mathcal{M}'}(s, a_2)$, then $Q^{\tilde{\pi}}_{\mathcal{M}}(s, a_1) > Q^{\tilde{\pi}}_{\mathcal{M}}(s, a_2)$. Otherwise, assume $Q^{\tilde{\pi}}_{\mathcal{M}}$ is continuous function with respect to $\gamma$, when $\gamma$ is sufficiently close to $1$, if  $Q^{\tilde{\pi}}_{\mathcal{M}'}(s, a_1) > Q^{\tilde{\pi}}_{\mathcal{M}'}(s, a_2)$, then $Q^{\tilde{\pi}}_{\mathcal{M}}(s, a_1) > Q^{\tilde{\pi}}_{\mathcal{M}}(s, a_2)$.
\end{lemma}

\begin{proof}
	According to the Bellman expectation equation of $Q^{\tilde{\pi}}_{\mathcal{M}}(s, a)$, we get the following equation by subtracting $\Psi(s)$ from both sides simultaneously (Note that $s'$ in the following equation is uniquely determined by $s$ and $a$):
	\begin{equation}
		\begin{aligned}
			\label{eqn:diff}
			Q^{\tilde{\pi}}_{\mathcal{M}}(s, a) - \Psi(s) &= r(s, a, s') + \gamma \mathbb{E}_{a' \sim \tilde{\pi}} \left[ Q^{\tilde{\pi}}_{\mathcal{M}}(s', a') \right] - \Psi(s) \\
			&= r(s, a, s') + \Psi(s') - \Psi(s) + \gamma \mathbb{E}_{a' \sim \tilde{\pi}} \left[ Q^{\tilde{\pi}}_{\mathcal{M}}(s', a') - \Psi(s') \right] + (\gamma - 1)\Psi(s') \\
			&= r'(s, a, s') + \gamma \mathbb{E}_{a' \sim \tilde{\pi}} \left[ Q^{\tilde{\pi}}_{\mathcal{M}}(s', a') - \Psi(s') \right] + (\gamma - 1)\Psi(s').
		\end{aligned}
	\end{equation}
	The Bellman Expectation equation for $ Q^{\tilde{\pi}}_{\mathcal{M}'}(s, a)$ is as follows:
	\begin{equation}
		\label{eqn:bellman_M'}
		Q^{\tilde{\pi}}_{\mathcal{M}'}(s, a) = r'(s, a, s') + \gamma \mathbb{E}_{a' \sim \tilde{\pi}} \left[ Q^{\tilde{\pi}}_{\mathcal{M}'}(s', a') \right].
	\end{equation}
	$1^\circ$ \,\, When $\gamma = 1$, \eqref{eqn:diff} can be written in the following form:
	\begin{equation}
		\label{eqn:bellman_M}
		Q^{\tilde{\pi}}_{\mathcal{M}}(s, a) - \Psi(s) = r'(s, a, s') + \gamma \mathbb{E}_{a' \sim \tilde{\pi}} \left[ Q^{\tilde{\pi}}_{\mathcal{M}}(s', a') - \Psi(s') \right].
	\end{equation}
	Compare \eqref{eqn:bellman_M'} and \eqref{eqn:bellman_M}, according to Potential-based Reward Shaping theory \cite{Ng-1999-Policy}, it can be known that:
	\begin{equation}
		Q^{\tilde{\pi}}_{\mathcal{M}'}(s, a) = Q^{\tilde{\pi}}_{\mathcal{M}}(s, a) - \Psi(s).
	\end{equation}
	When $Q^{\tilde{\pi}}_{\mathcal{M}'}(s, a_1) > Q^{\tilde{\pi}}_{\mathcal{M}'}(s, a_2)$, then there is:
	\begin{equation}
		\begin{aligned}
			Q^{\tilde{\pi}}_{\mathcal{M}}(s, a_1) - Q^{\tilde{\pi}}_{\mathcal{M}}(s, a_2) &= \left(Q^{\tilde{\pi}}_{\mathcal{M}'}(s, a_1) + \Psi(s)\right) - \left(Q^{\tilde{\pi}}_{\mathcal{M}'}(s, a_2) + \Psi(s)\right) \\
			&= Q^{\tilde{\pi}}_{\mathcal{M}'}(s, a_1) - Q^{\tilde{\pi}}_{\mathcal{M}'}(s, a_2) > 0.
		\end{aligned}
	\end{equation}
	So there is $Q^{\tilde{\pi}}_{\mathcal{M}}(s, a_1) > Q^{\tilde{\pi}}_{\mathcal{M}}(s, a_2)$.
	
	$2^\circ$ \,\, According to \eqref{eqn:diff} and \eqref{eqn:bellman_M'},we can split the $Q^{\tilde{\pi}}_{\mathcal{M}}(s, a) - \Psi(s)$ into the following form:
	\begin{equation}
		\begin{aligned}
			\label{eqn:19}
			Q^{\tilde{\pi}}_{\mathcal{M}}(s, a) - \Psi(s) &= \mathbb{E}_{\mathbb{P},\tilde{\pi}}\left[\left.\sum^\infty_{t=0}{\gamma^t \left(r'(s_t,a_t,s_{t+1}) + (\gamma - 1)\Psi(s_{t + 1})\right)}\right|s_0=s, a_0=a\right] \\
			&= \mathbb{E}_{\mathbb{P},\tilde{\pi}}\left[\left.\sum^\infty_{t=0}{\gamma^t r'(s_t,a_t,s_{t+1})}\right|s_0=s, a_0=a\right] + (\gamma - 1)\mathbb{E}_{\mathbb{P},\tilde{\pi}}\left[\left.\sum^\infty_{t=0}{\gamma^t\Psi(s_{t + 1})}\right|s_0=s, a_0=a\right] \\
			&= Q^{\tilde{\pi}}_{\mathcal{M}'}(s, a) + (\gamma - 1)\mathbb{E}_{\mathbb{P},\tilde{\pi}}\left[\left.\sum^\infty_{t=0}{\gamma^t\Psi(s_{t + 1})}\right|s_0=s, a_0=a\right].
		\end{aligned}
	\end{equation}
	Let $q^{\tilde{\pi}}(s, a) = \mathbb{E}_{\mathbb{P},\tilde{\pi}}\left[\left.\sum^\infty_{t=0}{\gamma^t\Psi(s_{t + 1})}\right|s_0=s, a_0=a\right] $, the following equation can be obtained:
	\begin{equation}
		\begin{aligned}
			\label{eqn:comp}
			Q^{\tilde{\pi}}_{\mathcal{M}}(s, a_1) - Q^{\tilde{\pi}}_{\mathcal{M}}(s, a_2) &= \left(Q^{\tilde{\pi}}_{\mathcal{M}}(s, a_1) - \Psi(s)\right) - \left(Q^{\tilde{\pi}}_{\mathcal{M}}(s, a_2) - \Psi(s)\right) \\
			&= Q^{\tilde{\pi}}_{\mathcal{M}'}(s, a_1) - Q^{\tilde{\pi}}_{\mathcal{M}'}(s, a_2) + (\gamma - 1) \left( q^{\tilde{\pi}}(s, a_1) - q^{\tilde{\pi}}(s, a_2) \right).
		\end{aligned}
	\end{equation}
	Since $Q^{\tilde{\pi}}_{\mathcal{M}}$ is continuous with respect to $\gamma$, then $Q^{\tilde{\pi}}_{\mathcal{M}}(s, a_1) - Q^{\tilde{\pi}}_{\mathcal{M}}(s, a_2)$ is also continuous with respect to $\gamma$. When $Q^{\tilde{\pi}}_{\mathcal{M}'}(s, a_1) > Q^{\tilde{\pi}}_{\mathcal{M}'}(s, a_2)$, taking the limit on both sides of \eqref{eqn:comp} has:
	\begin{equation}
		\lim_{\gamma \rightarrow 1^-} Q^{\tilde{\pi}}_{\mathcal{M}}(s, a_1) - Q^{\tilde{\pi}}_{\mathcal{M}}(s, a_2) = Q^{\tilde{\pi}}_{\mathcal{M}'}(s, a_1) - Q^{\tilde{\pi}}_{\mathcal{M}'}(s, a_2) > 0.
	\end{equation}
	By the order-preserving nature of the limit it follows that there exists $\delta > 0$ such that when $\gamma > 1 - \delta$, there is $Q^{\tilde{\pi}}_{\mathcal{M}}(s, a_1) > Q^{\tilde{\pi}}_{\mathcal{M}}(s, a_2)$
\end{proof}

\begin{theorem} (Consistency and Covergence)
	Assume $Q^{\tilde{\pi}}_{\mathcal{M}}$ is continuous function with respect to $\gamma$, when $\gamma$ is sufficiently close to $1$, the optimal policy obtained in MDP $\mathcal{M}'$  is also the optimal policy in MDP $\mathcal{M}$. The policy $\pi'$ obtained from the training of the RL algorithm has the following error bounds with the optimal policy, where $\pi^*$ is the optimal policy on $\mathcal{M}$, $\epsilon = \left\|  \mathcal{T}Q^{\pi'}_{\mathcal{M}} - Q^{\pi'}_{\mathcal{M}} \right\|_\infty$, $\mathcal{T}$ is the Bellman Optimal Operator:
	\begin{equation}
		\left\|  V^{\pi'}_{\mathcal{M}} - V^{\pi^*}_{\mathcal{M}} \right\|_\infty \leq \frac{2\gamma\epsilon}{1 - \gamma}. 
	\end{equation}
\end{theorem}
\begin{proof}
	The same initialization policy is used to optimize the policy in each of the two MDPs using the Policy Improvement Theorem \cite{Sutton-2018-Reinforcement}, and the optimal policies are obtained after continuous policy iterations. The action with the largest Q-function value will be selected in each state in the first iteration, and according to the Order Preserving Lemma, it is known that the action with the largest value is the same in both MDPs (if more than one action have the largest value, we assume that the same action will be selected in both MDPs). So after the first iteration, the same policy will still be obtained in both MDPs. By mathematical induction, we can see that after the same number of iterations, we will still get the same policy. Therefore, after iterating the same number of times to converge to the optimal policy, there will be the same optimal policy in MDP $M$ and MDP $M'$.
	
	By Theorem 4.1 in \citet{Williams-1993-Tight}, it follows that $\left\|  V^{\pi'}_{\mathcal{M}} - V^{\pi^*}_{\mathcal{M}} \right\|_\infty \leq \frac{2\gamma\epsilon}{1 - \gamma}$.
\end{proof}

\section{Annealing Exploration}
\label{app:annealing_exploration}

The trade-off between exploration and exploitation is a crucial challenge in RL \cite{Sutton-2018-Reinforcement} lying also in \method. Exploration is necessary to try out new responses that may hit the correct answer of a given question. Exploitation ensures that the policy select the best option based on currently explored experiences to maximize the long-term benefits.

\method{} trains the LLM in an on-policy manner where experience sampling is conducted using the latest version of its updating parametric policy. The amount of randomness in the response depends on the generation configuration of the current policy, especially the temperature $\tau$. Large $\tau$ enhances randomness by spreading the action probability distribution while small $\tau$ brings more certainty by sticking the most possible action out. Thus, analogous to the decreasing exploration rate in the $\epsilon$-greedy policy \cite{Mnih-2015-Humanlevel}, we propose an \emph{annealing exploration} strategy where the temperature for experience generation starts from a large value and gradually decreases, to control the trade-off between exploration and exploitation. Specifically, the temperature in training step $i$ is defined as
\begin{equation}
	\tau(i) = (1 - \frac{i}{N}) \tau^{\text{initial}} + \frac{i}{N} \tau^{\text{final}}, 
\end{equation}
where $\tau^{\text{initial}} > \tau^{\text{final}}$, and $\tau$ anneals linearly from $\tau^{\text{initial}}$ to $\tau^{\text{final}}$ over $N$ total training steps.

\section{Implementation Details and Hyperparameters}
\label{app:implementation}

\begin{algorithm2e}[ht]
	\SetKwInOut{Input}{Input}\SetKwInOut{Output}{Output}\SetKwInOut{Initialize}{Initialize}
	\SetKwRepeat{Do}{do}{while}
	
	\Input{counterfactual dataset $\sD = \{ (q, c^s, y^s, y^o) \}$, supervised fine-tuned policy $\pi^{\text{SFT}}$, max token number $T$, temperature scheduler $\tau(\cdot)$}
	\Initialize{policy network $\pi_\theta \leftarrow \pi^{\text{SFT}}$, value network $Q_\phi$}
	
	\Repeat{convergence}{
		$i \leftarrow 0$\\
		sample a batch of data $\sD_i$ from $\sD$\\
		\For{each $(q, c^s, y^s, y^o)$ in $\sD_i$}{
			$t \leftarrow 0$\\
			\Repeat{$a_t$ is end-of-text token {\bfseries or} $t = T$}{
				sample $a_t \sim \pi_\theta(a_t|q, c^s, a_{<t})$ with $\tau(i)$\\
				compute $r_t$ with \cref{eqn:reward}\\
				$t \leftarrow t + 1$\\
			}
		}
		\For{each gradient step}{
			PPO update on $\theta$ and $\phi$
		}
		$i \leftarrow i + 1$
	}
	\caption{\method{}}
	\label{alg:main}
\end{algorithm2e}

\begin{algorithm2e}
	\SetKwFunction{isCollapsed}{isCollapsed}
	\SetKwProg{Fn}{Function}{:}{}
	\Fn{\isCollapsed{$s$, $T$}}{
		\KwResult{Whether or not string $s$ is collapsed under repetition tolerance $T$}
		\For{each substring $s_{i:j}$ in $s$}{
			\For{$k$ in $1:\frac{j-i}{2}$}{
				\If{$(j-i) \bmod k = 0$}{
					candidate repetition mode $r \leftarrow s_{i:i+k}$\\
					candidate period $t \leftarrow \frac{j-i}{k}$\\
					\If{$t \geq T \wedge s_{i:j} = r * t$}{
						\Return \textbf{True}
					}
				}
			}
		}
		\Return \textbf{False}
	}
	\caption{Collapse Detection}
	\label{alg:collapse_detection}
\end{algorithm2e}

We implement \method{} based on the DeepSpeed-Chat framework \cite{Yao-2023-DeepSpeedChat} with LoRA \cite{Hu-2021-LoRA}. Experiments involving Llama-2-13b-chat or larger critic models are conducted on 8$\times$A100-80G GPUs. Other experiments are conducted on 8$\times$3090 GPUs. We use the Adam optimizer. The gradient checkpointing is enabled and we use ZeRO stage 2 for trainable models. The pipeline of \method{} is depicted in \cref{alg:main}. The collapse detection algorithm is described in \cref{alg:collapse_detection}.

We list hyperparameters in \cref{tab:hyperparameters}. All hyperparameters remain the same in the alignment of all pre-trained models.

\begin{table}[H]
	\renewcommand{\arraystretch}{1.1}
	\centering
	\caption{Hyperparameters in \method.}
	\label{tab:hyperparameters}
	\vskip 0.15in
	\begin{tabular}{l l| l }
		\toprule
		\multicolumn{2}{l|}{Parameter} &  Value\\
		\midrule
		\multicolumn{2}{l|}{\it{Training}}& \\
		& $\alpha_0$ & 3\\
		& $\alpha_1$ & 1\\
		& $\alpha_2$ & 0.05\\
		& $\alpha_3$ & 2\\
		& $\gamma$ & 1.0\\
		& $\lambda$ & 0.95\\
		& optimizer & Adam\\
		& training batch size & 8\\
		& max prompt sequence length & 1024\\
		& actor learning rate & $9.65 \times 10^{-6}$\\
		& critic learning rate & $5 \times 10^{-6}$\\
		& actor weight decay & $0.1$ \\
		& critic weight decay & $0.1$ \\
		& training epochs & 1\\
		& learning rate scheduler & cosine \\
		& warmup steps & 100 \\
		& actor lora dim & 64\\
		& actor lora learning rate & $5 \times 10^{-4}$\\
		& critic lora dim & 64\\
		& critic lora learning rate & $5 \times 10^{-4}$\\
		\midrule
		\multicolumn{2}{l|}{\it{Generation in training}}& \\
		& max new token number & 64 \\
		& do sample & True\\
		& top-p & 1.0\\
		& number of beams & 1\\
		& repetition penalty & 1.2\\
		& initial temperature $\tau^{\text{initial}}$ & 2.0 \\
		& final temperature $\tau^{\text{final}}$ & 0.0 \\
		& temperature decay type & linear \\
		\midrule
		\multicolumn{2}{l|}{\it{Generation in evaluation}}& \\
		& max new token number & 64 \\
		& do sample & False\\
		& number of beams & 1\\
		& repetition penalty & 1.2\\
		\bottomrule
	\end{tabular}
	\vskip -0.1in
\end{table}

\newpage
\section{Supplementary Experimental Results}

\subsection{Visualization of other statistics in the training process}
\begin{figure*}[ht]
	\vskip 0.2in
	\begin{center}
		\centerline{\includegraphics[width=\textwidth]{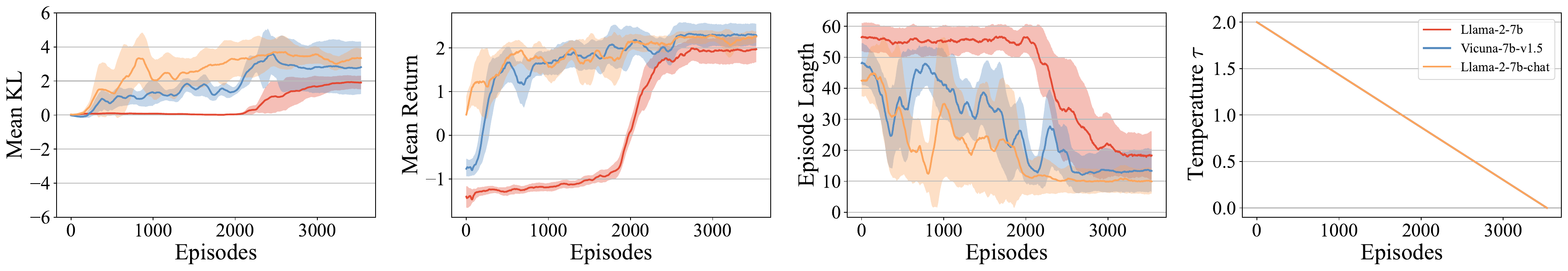}}
		\caption{Curves of token mean KL, token mean return, episode length and $\tau$ in the training procedure of \method{} on Natural Questions. The solid curves represent the mean value and the shaded areas correspond to the standard deviations of 3 runs with different seeds. The curve of token mean KL, token mean return and episode length are smoothed by averaging within a sliding window.}
		\label{fig:training_curves_rl_info}
	\end{center}
	\vskip -0.2in
\end{figure*}

\label{app:supp_results}
We visualize the curves of token mean KL, token mean return, episode length and $\tau$ in \cref{fig:training_curves_rl_info} for a deeper understanding of the training process.

\subsection{Alignment to the parametric knowledge}

\begin{table*}[ht]
	\centering
	\caption{The performance of Llama-2-7b-chat aligned to its parametric knowledge. Ours (original) and Ours (modified) use $r_t^{\text{trust}}$ in \cref{eqn:reward_trust} and \cref{eqn:reward_trust_both}, respectively.}
	\vskip 0.15in
	\centering
	\small
	\newcolumntype{C}{>{\centering\arraybackslash}p{2.3cm}}
	\begin{tabular}{lCCCC}
			\toprule
			Model & $p^s$ (\%) $\uparrow$ & $p^o$ (\%) $\uparrow$ & $p^{both}$ (\%) $\uparrow$ & $M_R$ (\%) $\uparrow$ \\
			\midrule
			Pre-trained & 46.00 & 29.68 & 1.47 & 40.00 \\
			Ours (original) & 94.92 & 0.67 & 0.45 & 0.70 \\ 
			Ours (modified) & 40.37 & 51.74 & 7.62 & 61.23 \\
			\bottomrule
		\end{tabular}
		\label{tab:align_both}
	\vskip -0.1in
\end{table*}

An assumption for \method{} is that the retriever is good enough to return contexts containing right answers for given questions. However, this assumption does not always hold true. To provide a potential solution for trade-off between parametric and contextual knowledge, we conduct experiment on probing parametric knowledge of LLMs via \method. Specifically, we modify the reward in \cref{eqn:reward_trust} to
\begin{equation}
	\label{eqn:reward_trust_both}
	r_t^{\text{trust}} =
	\begin{cases}
		+\alpha_0 & \text{if } s_{t+1} \text{ ends} \wedge (y^s \subseteq s_{t+1} \vee y^o \subseteq s_{t+1}) \\
		-\alpha_1 & \text{if } s_{t+1} \text{ ends} \wedge y^s \nsubseteq s_{t+1} \wedge y^o \nsubseteq s_{t+1} \\
		0 & \text{otherwise} \\
	\end{cases},
\end{equation}
with which the model is also rewarded when the original answer appears in the response. We conduct experiment with Llama-2-7b-chat and report the proportion of samples whose responses contain the substituted, original and both answers and memorization ration in \cref{tab:align_both}. It can be observed that $p^o$, $p^{both}$ and $M_R$ significantly improves with a simple modification on the reward coefficients, indicating \method's encouraging effect on recalling much more parametric knowledge. We reckon this ability a potential solution for the imperfect retriever and will take a deep dive into this research question in our future work.

In addition, it is worth noting that $p^s$ decreases slightly after the modified alignment. We check the exploration process in the training of LLM and find few responses in the collected experiences including both the contextual and parametric answers. A possible explanation is that pre-trained Llama-2-7b-chat is too weak to generate experience that includes both answers and to be reinforced. And thus reaching a response mode produce responses containing both the original and substituted answers is hard for small models.

\newpage
\section{Prompts}
\label{app:prompts}

\subsection{Instruction-based prompt for free-form QA:}
\begin{quote}
	\begin{tcolorbox}[size=normal,opacityfill=0.05]
	\begin{em}
		Instruction: answer the question based on the given context.\\
		Q: \{question\}?\\
		Context: \{document\}\\
		A:
	\end{em}
	\end{tcolorbox}
\end{quote}

\subsection{Opinion-based prompt for free-form QA:}
\begin{quote}
	\begin{tcolorbox}[size=normal,opacityfill=0.05]
	\begin{em}
		Bob said, ``\{document\}''\\
		Q: \{question\} in Bob's opinion?\\
		A:
	\end{em}
	\end{tcolorbox}
\end{quote}

\subsection{Instruction-based prompt for multi-choice QA:}
\begin{quote}
	\begin{tcolorbox}[size=normal,opacityfill=0.05]
	\begin{em}
		According to the given information, choose the best choice from the following options.\\
		\\
		Information:\\
		\{document\}\\
		\\
		Question:\\
		\{question\}\\
		\\
		Options:\\
		A. \{option A\}\\
		B. \{option B\}\\
		C. \{option C\}\\
		\\
		Answer:
	\end{em}
	\end{tcolorbox}
\end{quote}

\subsection{Opinion-based prompt for multi-choice QA:}
\begin{quote}
	\begin{tcolorbox}[size=normal,opacityfill=0.05]
	\begin{em}
		According to what Bob said, choose the best choice that is in agreement with Bob from the following options.\\
		\\
		Bob said: ``\{document\}''\\
		\\
		Question:\\
		\{question\} in Bob's opinion?\\
		\\
		Options:\\
		A. \{option A\}\\
		B. \{option B\}\\
		C. \{option C\}\\
		\\
		Answer:
	\end{em}
	\end{tcolorbox}
\end{quote}

\newpage
\section{Examples}
\label{app:examples}

We present several test samples and corresponding responses of models after \method{} and SFT. We denote the original and substituted answers with red and blue font in the context, respectively.

\begin{quote}
	\begin{tcolorbox}[size=normal,opacityfill=0.05]
		\textbf{Question:}\\
		who sings it 's my party and i cry if i want to?\\
		\\
		\textbf{Original context:}\\
		\textless P\textgreater `` It 's My Party '' is a pop song recorded by multiple artists since the 1960s . In 1963 , American singer \textred{Lesley Gore} 's version hit \# 1 on the pop and rhythm and blues charts in the United States . It was the first hit single for producer Quincy Jones . \textless /P\textgreater\\
		\\
		\textbf{Input:}\\
		Instruction: answer the question based on the given context.\\
		Q: who sings it 's my party and i cry if i want to?\\
		Context: \textless P\textgreater `` It 's My Party '' is a pop song recorded by multiple artists since the 1960s . In 1963 , American singer \textblue{Mariah Carey} 's version hit \# 1 on the pop and rhythm and blues charts in the United States . It was the first hit single for producer Quincy Jones . \textless /P\textgreater\\
		A:\\
		\\
		\textbf{Response of ChatGPT:}\\
		Lesley Gore sings ``It's My Party.''\\
		\\
		\textbf{Response of GPT-4:}\\
		Based on the context provided, it appears there is a mistake concerning the artist. ``It's My Party'' was made famous by American singer Lesley Gore in 1963, not Mariah Carey. This song was indeed the first hit single for producer Quincy Jones. Mariah Carey, to the best of publicly available knowledge as of my last update in 2023, has not recorded a version of ``It's My Party.''\\
		\\
		\textbf{Response of Llama-2-7b with \method:}\\
		Mariah Carey\textless br /\textgreater\\
		\\
		\textbf{Response of Vicuna-7b-v1.5 with \method:}\\
		Mariah Carey!\\
		\\
		\textbf{Response of Llama-2-7b-chat with \method:}\\
		Mariah Carey\\
		\\
		\textbf{Response of Llama-2-13b-chat with \method:}\\
		The artist who sang ``It's My Party'' is Mariah Carey.\\
		\\
		\textbf{Response of Llama-2-7b with SFT:}\\
		Mariah Carey\\
		\\
		\textbf{Response of Vicuna-7b-v1.5 with SFT:}\\
		Mariah Carey\\
		\\
		\textbf{Response of Llama-2-7b-chat with SFT:}\\
		Mariah Carey\\
		\\
		\textbf{Response of Llama-2-13b-chat with SFT:}\\
		Mariah Carey
	\end{tcolorbox}
\end{quote}

\begin{quote}
	\begin{tcolorbox}[size=normal,opacityfill=0.05]
		\textbf{Question:}\\
		when did australia 2 win the america 's cup?\\
		\\
		\textbf{Original context:}\\
		\textless P\textgreater The \textred{1983} America 's Cup was the occasion of the first winning challenge to the New York Yacht Club , which had successfully defended the cup over a period of 132 years . An Australian syndicate representing the Royal Perth Yacht Club fielded the Australia II , skippered by John Bertrand , against defender Liberty , skippered by Dennis Conner . Australia II won the match races to win the America 's Cup , ending the longest winning streak in sporting history and ending U.S. domination of the racing series . \textless /P\textgreater\\
		\\
		\textbf{Input:}\\
		Instruction: answer the question based on the given context.\\
		Q: when did australia 2 win the america 's cup?\\
		Context: \textless P\textgreater The \textblue{1952} America 's Cup was the occasion of the first winning challenge to the New York Yacht Club , which had successfully defended the cup over a period of 132 years . An Australian syndicate representing the Royal Perth Yacht Club fielded the Australia II , skippered by John Bertrand , against defender Liberty , skippered by Dennis Conner . Australia II won the match races to win the America 's Cup , ending the longest winning streak in sporting history and ending U.S. domination of the racing series . \textless /P\textgreater\\
		A:\\
		\\
		\textbf{Response of ChatGPT:}\\
		Australia II won the America's Cup in 1983.\\
		\\
		\textbf{Response of GPT-4:}\\
		The context provided contains conflicting information. The 1952 date mentioned is incorrect because Australia II won the America's Cup in 1983.\\
		\\
		\textbf{Response of Llama-2-7b with \method:}\\
		1952\textless br /\textgreater\\
		\\
		\textbf{Response of Vicuna-7b-v1.5 with \method:}\\
		1952!\\
		\\
		\textbf{Response of Llama-2-7b-chat with \method:}\\
		1952\\
		\\
		\textbf{Response of Llama-2-13b-chat with \method:}\\
		In 1952, Australia II won the America's Cup.\\
		\\
		\textbf{Response of Llama-2-7b with SFT:}\\
		1952\\
		\\
		\textbf{Response of Vicuna-7b-v1.5 with SFT:}\\
		1952\\
		\\
		\textbf{Response of Llama-2-7b-chat with SFT:}\\
		1952\\
		\\
		\textbf{Response of Llama-2-13b-chat with SFT:}\\
		1952
	\end{tcolorbox}
\end{quote}

\begin{quote}
	\begin{tcolorbox}[size=normal,opacityfill=0.05]
		\textbf{Question:}\\
		how many episodes are in chicago fire season 4?\\
		\\
		\textbf{Original context:}\\
		\textless P\textgreater The fourth season of Chicago Fire , an American drama television series with executive producer Dick Wolf , and producers Derek Haas , Michael Brandt , and Matt Olmstead , was ordered on February 5 , 2015 , by NBC , and premiered on October 13 , 2015 and concluded on May 17 , 2016 . The season contained \textred{23} episodes . \textless /P\textgreater\\
		\\
		\textbf{Input:}\\
		Instruction: answer the question based on the given context.\\
		Q: how many episodes are in chicago fire season 4?\\
		Context: \textless P\textgreater The fourth season of Chicago Fire , an American drama television series with executive producer Dick Wolf , and producers Derek Haas , Michael Brandt , and Matt Olmstead , was ordered on February 5 , 2015 , by NBC , and premiered on October 13 , 2015 and concluded on May 17 , 2016 . The season contained \textblue{127} episodes . \textless /P\textgreater\\
		A:\\
		\\
		\textbf{Response of ChatGPT:}\\
		There are 23 episodes in Chicago Fire season 4.\\
		\\
		\textbf{Response of GPT-4:}\\
		The context provided states that the fourth season of ``Chicago Fire'' contained 127 episodes, which is likely an error as standard American television seasons typically consist of far fewer episodes. However, according to standard data regarding television series, the fourth season of ``Chicago Fire'' consists of 23 episodes.\\
		\\
		\textbf{Response of Llama-2-7b with \method:}\\
		127\textless br /\textgreater\\
		\\
		\textbf{Response of Vicuna-7b-v1.5 with \method:}\\
		127!\\
		\\
		\textbf{Response of Llama-2-7b-chat with \method:}\\
		127\\
		\\
		\textbf{Response of Llama-2-13b-chat with \method:}\\
		There are 127.\\
		\\
		\textbf{Response of Llama-2-7b with SFT:}\\
		127\\
		\\
		\textbf{Response of Vicuna-7b-v1.5 with SFT:}\\
		127\\
		\\
		\textbf{Response of Llama-2-7b-chat with SFT:}\\
		127\\
		\\
		\textbf{Response of Llama-2-13b-chat with SFT:}\\
		127
	\end{tcolorbox}
\end{quote}

\begin{quote}
	\begin{tcolorbox}[size=normal,opacityfill=0.05]
		\textbf{Question:}\\
		in which country bikram sambhat the official calender?\\
		\\
		\textbf{Original context:}\\
		\textless P\textgreater The Rana rulers of \textred{Nepal} made Vikram Samvat the official Hindu calendar in 1901 CE , which started as Samvat 1958 . In \textred{Nepal} , the new year begins with the first day of the month of Baishakh , which usually falls within the months of April -- May in the Gregorian calendar . The first day of the new year is passionately celebrated in a historical carnival that takes place every year in Bhaktapur , called Bisket Jatra.As before , from 2007 AD \textred{Nepal} Sambat is recognized as the national calender . \textless /P\textgreater\\
		\\
		\textbf{Input:}\\
		Instruction: answer the question based on the given context.\\
		Q: in which country bikram sambhat the official calender?\\
		Context: \textless P\textgreater The Rana rulers of \textblue{France} made Vikram Samvat the official Hindu calendar in 1901 CE , which started as Samvat 1958 . In \textblue{France} , the new year begins with the first day of the month of Baishakh , which usually falls within the months of April -- May in the Gregorian calendar . The first day of the new year is passionately celebrated in a historical carnival that takes place every year in Bhaktapur , called Bisket Jatra.As before , from 2007 AD \textblue{France} Sambat is recognized as the national calender . \textless /P\textgreater\\
		A:\\
		\\
		\textbf{Response of ChatGPT:}\\
		Nepal\\
		\\
		\textbf{Response of GPT-4:}\\
		The given context contains incorrect information because the Vikram Samvat is not associated with France. Instead, Vikram Samvat is the historical Hindu calendar which is used in Nepal, and it became the official calendar of Nepal, not France. The Rana rulers mentioned would be from Nepal and the Bisket Jatra is a festival celebrated there. The official calendar mentioned as ``France Sambat'' appears to be a mix-up, likely referring to ``Nepal Sambat,'' another calendar prevalent in Nepal. Thus, the correct country should be Nepal, not France.\\
		\\
		\textbf{Response of Llama-2-7b with \method:}\\
		France\textless br /\textgreater\\
		\\
		\textbf{Response of Vicuna-7b-v1.5 with \method:}\\
		France!\\
		\\
		\textbf{Response of Llama-2-7b-chat with \method:}\\
		France\\
		\\
		\textbf{Response of Llama-2-13b-chat with \method:}\\
		Based on the given context, the correct answer is France.\\
		\\
		\textbf{Response of Llama-2-7b with SFT:}\\
		France\\
		\\
		\textbf{Response of Vicuna-7b-v1.5 with SFT:}\\
		France\\
		\\
		\textbf{Response of Llama-2-7b-chat with SFT:}\\
		France\\
		\\
		\textbf{Response of Llama-2-13b-chat with SFT:}\\
		France
	\end{tcolorbox}
\end{quote}

\begin{quote}
	\begin{tcolorbox}[size=normal,opacityfill=0.05]
		\textbf{Question:}\\
		who is the actor who plays king joffrey?\\
		\\
		\textbf{Original context:}\\
		\textless P\textgreater \textred{Jack Gleeson} ( born 20 May 1992 ) is an Irish former actor , best known for his portrayal of Joffrey Baratheon in the HBO television series Game of Thrones . \textless /P\textgreater\\
		\\
		\textbf{Input:}\\
		Instruction: answer the question based on the given context.\\
		Q: who is the actor who plays king joffrey?\\
		Context: \textless P\textgreater \textblue{Brad Dourif} ( born 20 May 1992 ) is an Irish former actor , best known for his portrayal of Joffrey Baratheon in the HBO television series Game of Thrones . \textless /P\textgreater\\
		A:\\
		\\
		\textbf{Response of ChatGPT:}\\
		Jack Gleeson\\
		\\
		\textbf{Response of GPT-4:}\\
		The actor who plays King Joffrey is Brad Dourif.\\
		\\
		\textbf{Response of Llama-2-7b with \method:}\\
		Brad Dourif\textless br /\textgreater\\
		\\
		\textbf{Response of Vicuna-7b-v1.5 with \method:}\\
		Brad Dourif!\\
		\\
		\textbf{Response of Llama-2-7b-chat with \method:}\\
		Brad Dourif\\
		\\
		\textbf{Response of Llama-2-13b-chat with \method:}\\
		The actor who plays King Joffrey is Brad Dourif, who was born on May 20, 1992.\\
		\\
		\textbf{Response of Llama-2-7b with SFT:}\\
		Brad Dourif\\
		\\
		\textbf{Response of Vicuna-7b-v1.5 with SFT:}\\
		Brad Dourif\\
		\\
		\textbf{Response of Llama-2-7b-chat with SFT:}\\
		Brad Dourif\\
		\\
		\textbf{Response of Llama-2-13b-chat with SFT:}\\
		Brad Dourif
	\end{tcolorbox}
\end{quote}


\end{document}